%% file: main.tex
\documentclass[journal]{IEEEtran}
\usepackage{amsmath,amsfonts,amsthm,mathtools}
\usepackage[linesnumbered, ruled, vlined]{algorithm2e}
\makeatletter
\newcommand{\removelatexerror}{\let\@latex@error\@gobble}
\makeatother
\usepackage[table,dvipsnames,xcdraw]{xcolor}

\SetCommentSty{commentstyle}
\usepackage{array}
\usepackage{graphicx}
\usepackage{booktabs}
\usepackage{tabularx}
\usepackage{multirow}
\usepackage[hidelinks]{hyperref}
\hypersetup{colorlinks=true,linkcolor=blue,urlcolor=black,citecolor=blue}
\usepackage[caption=false,font=footnotesize]{subfig}
\usepackage{rotating}
\usepackage[export]{adjustbox}
\usepackage{bm}
\usepackage{cite}
\usepackage{orcidlink}
\DeclareMathAlphabet{\mathbbold}{U}{bbold}{m}{n}
\renewcommand{\IEEEQED}{\IEEEQEDopen}
\newtheorem{theorem}{Theorem}
\newtheorem{lemma}{Lemma}
\newtheorem{remark}{Remark}
\newtheorem{definition}{Definition}

\begin{document}
\title{KinematicRL: A Sim-to-Real Reinforcement Learning Framework For Social Navigation With Kinodynamic Feasibility}
\author{Zhiming~Xu$^{\orcidlink{0009-0005-3382-8474}}$,
Haodong~Yang$^{\orcidlink{0009-0007-9104-9430}}$,
Chengju~Liu$^{\orcidlink{0000-0001-7543-0855}}$,~\IEEEmembership{Member,~IEEE},
Qijun~Chen$^{\orcidlink{0000-0001-5644-1188}}$,~\IEEEmembership{Senior~Member,~IEEE},
and~Chenpeng~Yao$^{\orcidlink{0000-0002-6804-3539}}$\vspace{-5ex}%
\thanks{This work was supported by the National Natural Science Foundation of China 
under Grant 62403358, Grant 62333017 and Grant 62233013. 
\textit{(Corresponding author: Chenpeng Yao.)}}%
\thanks{Zhiming Xu is with the School of Computer Science and Technology, 
Tongji University, Shanghai 201804, China (e-mail: \href{mailto:2251804@tongji.edu.cn}{2251804@tongji.edu.cn}).}
\thanks{Haodong Yang, Qijun Chen and Chenpeng Yao are with the Department of Electronics and Information Engineering, Tongji University, Shanghai 201804, China (e-mail: \href{mailto:2230700@tongji.edu.cn}{2230700@tongji.edu.cn}; \href{mailto:qjchen@tongji.edu.cn}{qjchen@tongji.edu.cn}; \href{mailto:yaochenpeng@tongji.edu.cn}{yaochenpeng@tongji.edu.cn}).}
\thanks{Chengju Liu is with the Department of Electronics and Information 
Engineering, Tongji University, Shanghai 201804, China, and also with 
Shanghai Institute of Intelligent Science and Technology, Tongji University, 
Shanghai 201210, China (e-mail: \href{mailto:liuchengju@tongji.edu.cn}{liuchengju@tongji.edu.cn}).}
}

\markboth{IEEE TRANSACTIONS ON AUTOMATION SCIENCE AND ENGINEERING}%
{Xu \MakeLowercase{\textit{et al.}}: KinematicRL: A Sim-to-Real Reinforcement Learning Framework For Social Navigation With Kinodynamic Feasibility}

\maketitle
\input{contents/abstract.tex}
\input{contents/introduction.tex}
\input{contents/related_works.tex}
\input{contents/approach_pf.tex}
\input{contents/approach_ilqr.tex}
\input{contents/approach_clustering.tex}
\input{contents/approach_net.tex}
\input{contents/approach_drl.tex}
\input{contents/experiments.tex}
\input{contents/conclusion.tex}
\bibliographystyle{IEEEtran}
\bibliography{IEEEabrv.bib, ref.bib}

\begin{IEEEbiography}[{\includegraphics[width=1in,height=1.25in,clip,keepaspectratio]{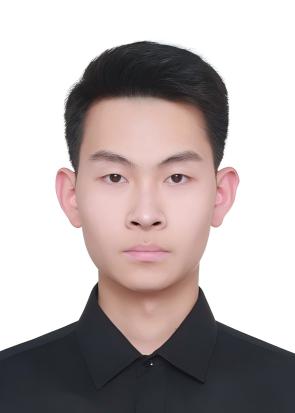}}]{Zhiming Xu} is expected to receive the B.Eng. degree in computer science from Tongji University, Shanghai, China, in 2026. He will pursue the M.S. degree in robotics with the Robotics Institute, Carnegie Mellon University, Pittsburgh, PA, USA. His research interests include deep reinforcement learning, robotics, and representation learning.
\end{IEEEbiography}

\vspace{11pt}

\begin{IEEEbiography}[{\includegraphics[width=1in,height=1.25in,clip,keepaspectratio]{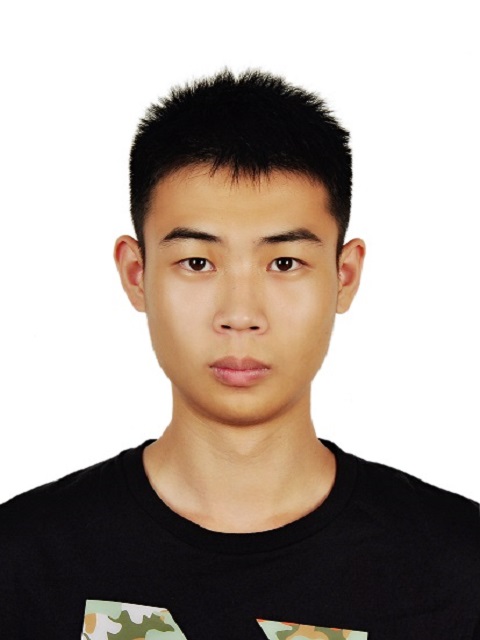}}]{Haodong Yang} received the B.S. degree in control science and engineering from Tongji University, Shanghai, China, in 2022, where he also received the M.S. degree in control science and engineering. His research interests include light detection and ranging (LiDAR) simultaneous localization, navigation of mobile robots, and deep reinforcement learning.
\end{IEEEbiography}

\vspace{11pt}

\begin{IEEEbiography}[{\includegraphics[width=1in,height=1.25in,clip,keepaspectratio]{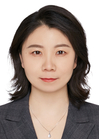}}]{Chengju Liu} (Member, IEEE) received the Ph.D. degree in control theory and control engineering from Tongji University, Shanghai, China, in 2011. From October 2011 to July 2012, she was with the BEACON Center, Michigan State University, East Lansing, MI, USA, as a Research Associate. From March 2011 to June 2013, she was a Postdoctoral Researcher with Tongji University, where she is currently a Professor with the Department of Control Science and Engineering, College of Electronics and Information Engineering, and a Chair Professor of Tongji Artificial Intelligence (Suzhou) Research Institute. She is also a Team Leader with the TJArk Robot Team, Tongji University. Her research interests include intelligent control, motion control of legged robots, and evolutionary computation.
\end{IEEEbiography}

\vspace{11pt}

\begin{IEEEbiography}[{\includegraphics[width=1in,height=1.25in,clip,keepaspectratio]{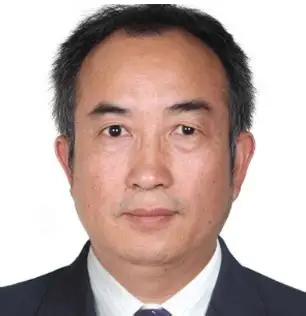}}]{Qijun Chen} (Senior Member, IEEE) received the B.S. degree in Automatic Control from Huazhong University of Science and Technology, Wuhan, China, in 1987, the M.S. degree in Control Engineering from Xi'an Jiaotong University, Xi'an China, in 1990 and the Ph.D. degree in Control Theory and Control Engineering from Tongji University, Shanghai, China,in 1999, respectively. He was a Visiting Professor with the University of California, Berkeley, CA, USA, in 2008. He is currently a Professor in the College of Electronic and Information Engineering, Tongji University, Shanghai, China. His current research interests include network-based control systems and robotics.
\end{IEEEbiography}

\vspace{11pt}

\begin{IEEEbiography}[{\includegraphics[width=1in,height=1.25in,clip,keepaspectratio]{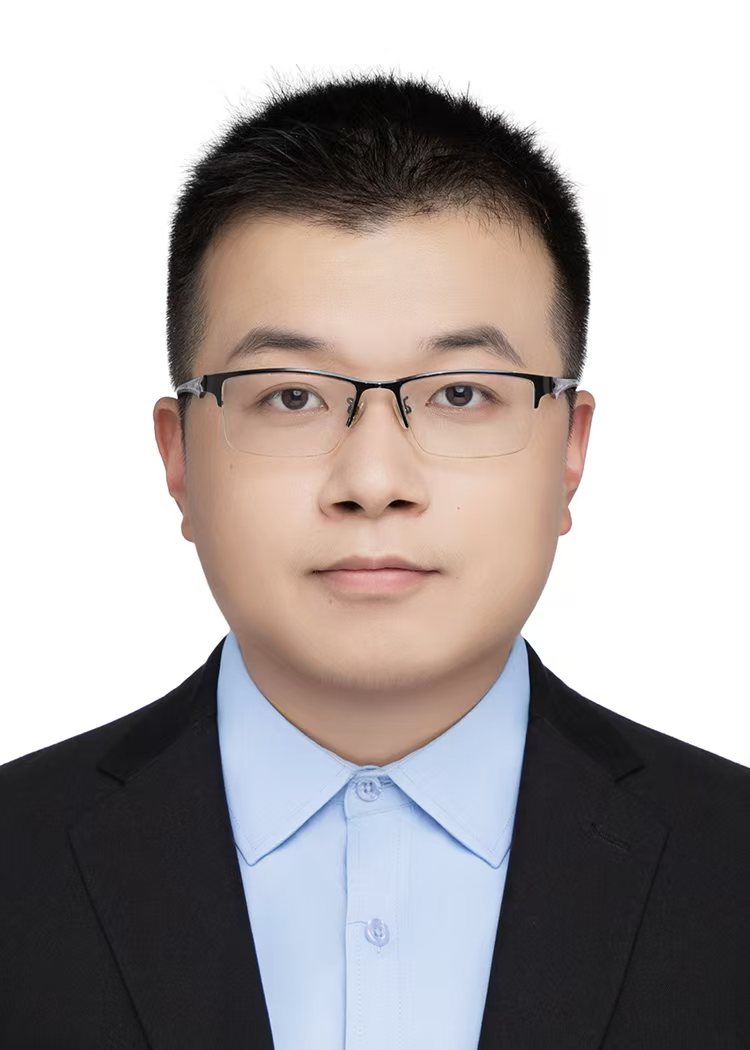}}]{Chenpeng Yao} received the Ph.D. degree in Control Theory and Control Engineering from Tongji University, Shanghai, China, in 2022. From April 2022 to September 2025, he was a Postdoctoral Researcher with Tongji University, where he is currently an Assistant Professor with the Department of Control Science and Engineering, College of Electronics and Information Engineering. His field of study is the locomotion of biped robot, navigation of mobile robot and reinforcement learning.
\end{IEEEbiography}

\vfill

\clearpage
\appendices
\input{contents/appendix/appendix_tracking_error.tex}
\input{contents/appendix/appendix_maximum_entropy.tex}
\end{document}

%% file: contents/abstract.tex
\begin{abstract}
Deep Reinforcement Learning (DRL) has shown promise for social navigation, yet its real-world deployment remains hindered by a persistent sim-to-real gap arising from simplified first-order dynamics and context-specific human state estimation pipelines. This work presents a unified framework that addresses these limitations to produce dynamically feasible navigation policies suitable for real-world deployment. First, theoretical analysis reveals that tracking error between simulated and actual robot position decays exponentially with increased control order, motivating the use of higher-order control inputs as DRL action space. A second-order control formulation tailored to differential drive robots is developed, complemented by a stochastic iterative Linear Quadratic Regulator (iLQR) that pretrains the policy via a divergence minimization objective. Second, to avoid the added system complexity of camera-LiDAR fusion, a cluster-based human tracking pipeline using only 2D LiDAR is introduced. Human detections are associated according to both spatial proximity and velocity similarity, enabling reliable differentiation of nearby pedestrians and yielding stable velocity estimates through temporal aggregation. Third, we introduce an unbiased residual gating block to balance reaction- and memory-based behaviors while handling time-varying crowd sizes, both critical for social navigation. The resulting policy, \textit{KinematicRL}, consistently improves kinematic performance and adapts to varying number of detected humans. Experiments in real-world environments demonstrate that, when combined with the proposed tracking pipeline, KinematicRL can be deployed on a real differential drive robot with minimal modifications.
\end{abstract}

\def\abstractname{Note to Practitioners}
\begin{abstract}
This work is motivated by the potential of deep reinforcement learning in social navigation and the sim-to-real gap widely witnessed during industrial applications. A comprehensive framework is proposed, ensuring dynamic feasibility in simulation, while effectively providing robust position and velocity estimates from 2D LiDAR data. By incorporating second-order control and warm-starting the policy with a stochastic iterative Linear Quadratic Regulator, our approach ensures dynamic feasibility and enhances kinematic performance for differential drive robots. The proposed cluster-based tracking pipeline relies solely on 2D LiDAR, utilizing online clustering with velocity similarity for robust matching, and historical averaging for smooth velocity estimation. Additionally, a transformer architecture is proposed to handle varying number of detected humans resulting from sensor limitations. Real-world experiments further suggest that a relatively high control rate effectively mitigates jerk-related approximation errors in second-order models, enabling the policy to be directly transferred to a real physical robot with minimal modifications. The proposed methods are computationally efficient and require only standard 2D LiDAR sensors, making them accessible for real-world deployment in industrial and service robotics settings.
\end{abstract}


%% file: contents/introduction.tex
\section{Introduction}
\IEEEPARstart{S}{ocial} navigation involves the movement of mobile robots in pedestrian-rich environments (e.g., airports, hospitals), with key challenges arising from high density dynamic obstacles and uncertainties in human behavior. Deep Reinforcement Learning (DRL) yields the 
potential to model Human-Robot Interactions (HRI) through trial and error, enabling robots to navigate socially and cooperatively in crowded scenarios~\cite{cadrl},~\cite{sa-cadrl},~\cite{lstm-rl},~\cite{rl-with-pedestrian-prediction}. Most DRL-based crowd navigation methods are trained entirely in simulators, where they rely on simplified first-order dynamics and idealized agent-level information, and aim to transfer trained models to real robots with minimal modifications~\cite{rl-with-pedestrian-prediction},~\cite{sarl},~\cite{RMRL}. However, these assumptions rarely hold in practice, thus creating a significant sim-to-real gap: real robots cannot execute the instantaneous velocity changes implied by first-order models, and human states must be estimated from sensor measurements.
\par Despite extensive efforts to narrow the sim-to-real gap in social navigation, several key limitations persist. (1) Existing methods that enforce dynamically feasible velocity commands either (i) generate subgoals for a separate model-based controller to track~\cite{spline-LQR-waypoint},~\cite{rl-subgoal-for-mpc},~\cite{tracking-goal-with-low-level-controller} or (ii) constrain the action space to feasible actions~\cite{wheel-acceleration-for-action-space},~\cite{dwa-rl-dynamic-feasibility},~\cite{st-graph-and-kinematics},~\cite{cbf-to-project-rl-actions}. These approaches either require the design of an extra controller, or restrict the action space to handcrafted, platform-specific feasible commands. (2) Safety-critical social navigation requires accurate distance measurements from LiDAR, yet most human-state estimation pipelines rely on LiDAR-camera fusion~\cite{cadrl},~\cite{lstm-rl}, which introduces calibration drift and higher system complexity. (3) The number of detected humans varies over time due to sensor range limits and occlusions. While recent works employ attention mechanisms to process variable length inputs~\cite{rl-with-pedestrian-prediction},~\cite{st-graph-and-kinematics}, experiments reveal that the standard transformer architecture exhibits optimization instability in social navigation settings, particularly when trained with padded inputs.
\par This article addresses the sim-to-real gap in social navigation by targeting the three limitations outlined above. For problem (1), we provide theoretical analysis showing that tracking error between simulated and real-world positions decays exponentially as control order increases, motivating a unified strategy for enforcing dynamic feasibility in DRL---using higher-order control inputs as the action space. In practice, second-order control is shown to be a natural fit for differential drive robots, and a stochastic iterative Linear Quadratic Regulator (iLQR) tailored for social navigation contexts is introduced to pretrain the second-order navigation policy via a divergence minimization objective. To address shortcoming (2), we propose a cluster-based tracking pipeline using only 2D LiDAR. Human positions are detected using Li2Former~\cite{li2former}, after which a data association procedure groups points using a distance metric incorporating both spatial and velocity similarity, enabling reliable differentiation of nearby humans exhibiting distinct temporal behaviors. Each human is represented as its cluster centroid, which provides stable position and velocity estimates by aggregating historical cluster observations. For problem (3), we utilize Gated Recurrent Units (GRUs) to stabilize transformer in social navigation by balancing the residual and attention stream. Through modifications on both architectural design and training methodology, a gated spatio-temporal transformer is introduced, yielding robustness to variable crowd sizes and sensing noise.
\par Putting these together, second-order control inputs form the action space, with the proposed gated spatio-temporal transformer serving as the feature extractor. Stochastic iLQR is employed to pretrain the corresponding second-order policy, which is then further trained with reinforcement learning. The resulting policy, \textit{KinematicRL}, outperforms baseline models across all kinematic metrics. When integrated with the proposed LiDAR-based tracking pipeline, KinematicRL can be directly transferred to a real differential drive robot.
\par In summary, our main contributions are four-fold: 
\begin{enumerate}
    \item \textit{Higher-Order Control For Feasibility}: Provide theoretical justification for using higher-order control to reduce tracking error, and introduce a practical second-order input formulation that enforces dynamic feasibility for differential drive robots. We further design a stochastic iLQR in social navigation settings to effectively pretrain the second-order navigation policy.
    \item \textit{Cluster-Based Tracking Pipeline}: Propose a cluster-based human state estimation pipeline utilizing raw 2D point clouds to generate accurate positions and smooth velocity predictions suitable for real-world deployment.
    \item \textit{Gated Transformer For Crowd Adaptation}: Design a gated transformer architecture utilizing unbiased residual gating and extensive domain randomization to achieve robust generalization under varying crowd densities.
    \item \textit{Sim-to-Real Validation}: Extensive validation across simulation and real-world environments. User-friendly ROS nodes for tracking and planning are made publicly available to facilitate future research.
\end{enumerate}
\par The rest of the paper is organized as follows. Section~\ref{sec:related-works} reviews related work. Section~\ref{sec:approach} details the problem formulation, dynamic feasibility analysis, stochastic iLQR design, tracking pipeline, and the proposed network architecture. Section~\ref{sec:experiments} presents experiments in both simulation and real-world settings. Finally, Section~\ref{sec:conclusion} concludes the paper and outlines directions for future work.

%% file: contents/related_works.tex
\section{Related Work}\label{sec:related-works}
\begin{figure*}[!ht]
    \centering
    \includegraphics[width=1.0\textwidth]{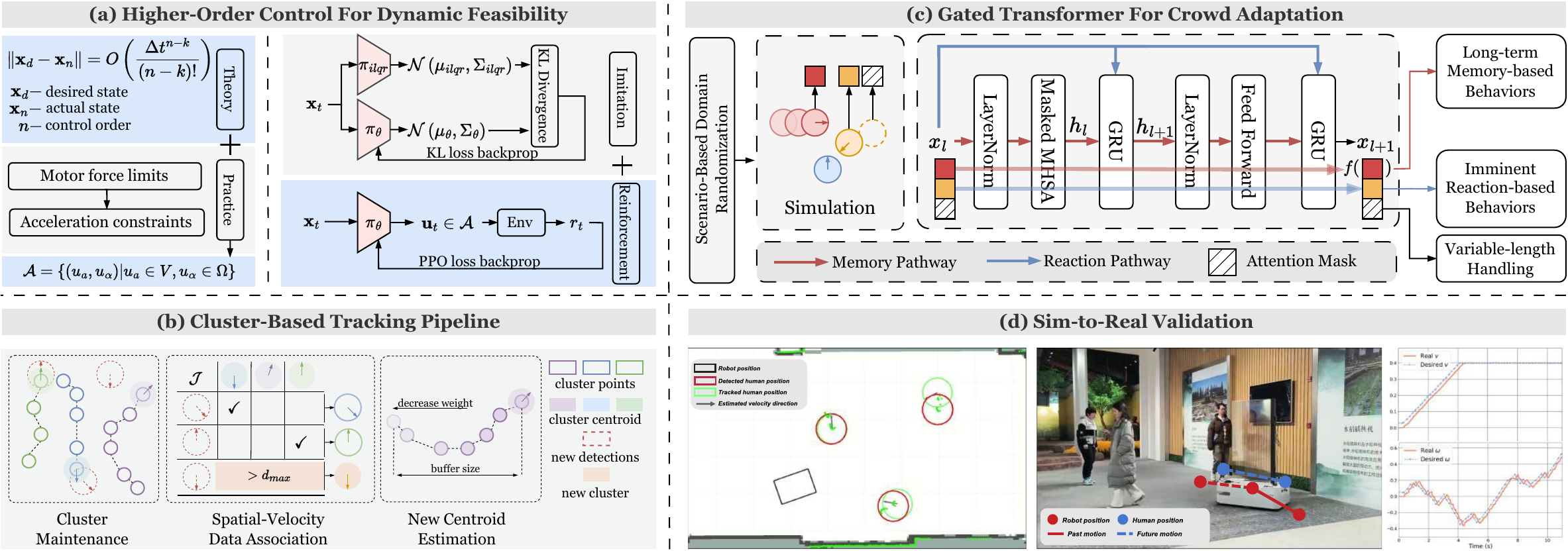}
    \caption{Overview of the proposed framework. (a) Higher-order control for dynamically feasible social navigation. A second-order control formulation is adopted as the action space, motivated by both theoretical analysis and practical considerations. A stochastic iLQR is designed to generate expert demonstrations and pretrain the second-order policy via a divergence minimization objective, after which the policy is further optimized using reinforcement learning. (b) Cluster-based tracking pipeline using only 2D LiDAR. Humans are viewed as centroids for their corresponding clusters (cluster maintenance). New detections are assigned to existing clusters (or new clusters are created) based on both spatial proximity and velocity similarity to the centroid (spatial-velocity data association). Cluster centroid velocities are estimated via historical averaging of detections stored in the buffer (new centroid estimation). (c) Gated spatio-temporal transformer architecture. The attention stream captures memory-based behaviors, while the residual stream supports reactive responses. Modified residual connections are introduced to improve stability under variable-length inputs arising from time-varying crowd sizes. (d) When integrated with the proposed tracking pipeline, the \textit{KinematicRL} policy demonstrates proactive collision avoidance and low-latency control in pedestrian-rich environments.}
    \label{fig:whole-process}
\end{figure*}
\subsection{Dynamic Feasibility in Social Navigation}
\par Geometric approaches for dynamic obstacle avoidance, such as Reciprocal Velocity Obstacle (RVO)~\cite{RVO}, Optimal Reciprocal Collision Aviodance (ORCA)~\cite{orca} and NonHolonomic-ORCA (NH-ORCA)~\cite{nh-orca}, reason over relative velocities of the robot and surrounding agents to avoid collisions; however, they are not well-suited for incorporating social distance constraints. The optimization-based formulation of Control Barrier Functions (CBFs) facilitates such integration, but recent methods tend to trigger avoidance maneuvers even when obstacle are distant~\cite{cbf-with-vo},~\cite{cbf-vo-nonholonomic}. In contrast, DRL-based methods hold stronger potential to model human-robot interactions through trial-and-error~\cite{cadrl},~\cite{sa-cadrl}. Nevertheless, most existing approaches neglect dynamic feasibility constraints and provide no guarantees that the generated velocities can be accurately tracked by real robots. Reliable tracking of commanded actions is critical for both safety and task success~\cite{rl-with-pedestrian-prediction},~\cite{sarl}, making it as important as action generation itself.
\par Recent efforts to encourage dynamic feasibility can be broadly categorized into three directions. (1) \emph{Incorporating model-based control:} These approaches employ model-based controllers to generate dynamically feasible actions, while training an RL policy to produce intermediate subgoals~\cite{spline-LQR-waypoint},~\cite{rl-subgoal-for-mpc},~\cite{tracking-goal-with-low-level-controller}. For example,~\cite{rl-subgoal-for-mpc} integrates Model Predictive Control (MPC) to track subgoals generated by a DRL policy. However, this approach incurs substantial computational overhead due to repeatedly solving non-convex optimization problems during training. (2) \emph{Action reformulation:} Another line of work enforces kinematic feasibility through action space design or post-hoc projection. This includes directly constraining the action space to feasible commands~\cite{wheel-acceleration-for-action-space},~\cite{dwa-rl-dynamic-feasibility},~\cite{st-graph-and-kinematics}, or projecting RL-generated actions into feasible sets via optimization~\cite{cbf-to-project-rl-actions, rl-weight-dwa}. For instance, Patel \textit{et al.}~\cite{dwa-rl-dynamic-feasibility} restrict the action space to a discrete set of dynamically reachable velocities. Similarly,~\cite{cbf-to-project-rl-actions} employs CBFs to project RL actions into a safe and feasible set. However, both discretization and post-hoc projection constrain the expressivity of the action space. (3) \emph{Reward shaping for feasibility:} A third direction encourages feasible behavior through reward design, such as incorporating regularization terms that penalize infeasible actions~\cite{RMRL},~\cite{soadrl},~\cite{reward-feasible-mobile-manipulation}. However, such methods typically require careful tuning and lack formal guarantees.
\par In summary, existing approaches either rely on auxiliary model-based controllers or impose handcrafted, platform-specific constraints on the action space. In contrast, this work proposes a unified framework that enforces dynamic feasibility directly within the learning formulation.

\subsection{Human Tracking Pipelines}
\par Safety-critical social navigation requires accurate distance measurements, for which 2D LiDAR has emerged as a practical choice for indoor scenarios. Most advancements in social 
navigation disentangle decision-making and perception~\cite{cadrl},~\cite{sa-cadrl},~\cite{rl-with-pedestrian-prediction},~\cite{sarl},~\cite{RMRL}, assuming complete awareness of human positions and velocities in simulation, thus necessitating a pedestrian tracking pipeline for real-world deployment. Existing tracking pipelines predominantly rely on context-specific LiDAR-camera fusion for human state estimation~\cite{cadrl},~\cite{lstm-rl}, which introduces calibration drift and additional system complexity. While broadly applicable, LiDAR-only pipelines often suffer from instability when humans are spatially close~\cite{rl-with-pedestrian-prediction},~\cite{RMRL}. Some approaches further depend on motion capture systems during deployment~\cite{wheel-acceleration-for-action-space}, limiting scalability and applicability.
\par For LiDAR-based position detection, DR-SPAAM~\cite{dr-spaam-detector} integrates historical time frames to achieve high recall rate, but often misidentifies background objects as humans. Yang \textit{et al.} identified this limitation and introduced Li2Former which achieves state-of-the-art precision when detecting human positions from 2D LiDAR point clouds~\cite{li2former}. For subsequent data association and tracking, SORT~\cite{SORT-iou-data-association} employs Intersection-over-Union (IoU) as the matching criterion, without assessing similarity across time dimension, leading to frequent identity switches. DeepSORT~\cite{DeepSORT-mahalanobis-data-association} extends this framework by incorporating Mahalanobis distance between state estimates, but its reliance on matrix inversion can introduce numerical instability in practice. Recognizing these limitations in existing tracking pipelines, we build upon recent advances in LiDAR-based human position detection to develop a tracking pipeline using only 2D LiDAR.

\subsection{Transformers in Social Navigation}
\par Social navigation requires capturing complex, non-linear interactions among agents whose number and behaviors vary over time, resulting in a history-dependent optimal policy~\cite{understanding-domain-randomization}. While Recurrent Neural Networks (RNNs) are commonly employed to provide memory to agents~\cite{lstm-rl},~\cite{rl-subgoal-for-mpc}, they struggle with long-term dependencies and forgetting, moreover their sequential processing mechanisms limit parallelization.
\par Transformers~\cite{transformer} utilize the attention mechanism to model pairwise social interactions and determine the relative importance of neighboring agents~\cite{rl-with-pedestrian-prediction},~\cite{st2}. Despite their expressivity, transformers face training challenges in reinforcement learning, where data non-stationarity amplify optimization instability~\cite{gru-transformers},~\cite{attention-for-meta-rl},~\cite{tfixup-initialization}. Parisotto \textit{et al.}~\cite{gru-transformers} address this issue with the Gated Transformer-XL (GTrXL), which employs GRUs to modulate the residual and attention streams, demonstrating superior performance over RNNs on long-horizon memory tasks. Chen \textit{et al.}~\cite{sarl} replace self-attention with global attention to reduce training complexity, but this loss of fidelity can degrade performance in dense crowds. Another stream of work applies attention mechanisms over explicitly constructed interaction graphs~\cite{rl-with-pedestrian-prediction},~\cite{wheel-acceleration-for-action-space},~\cite{st-graph-and-kinematics}. While graph-based formulations encode structured inductive biases, they are sensitive to noisy state estimates that change graph topology. To further facilitate the use of transformers in social navigation, we design a gated spatio-temporal transformer architecture that explicitly stabilizes training under noisy state estimates and time-varying crowd sizes.

%% file: contents/approach_pf.tex
\section{Approach}\label{sec:approach}
\subsection{Problem Formulation}\label{sec:problem-formulation}
\par We formulate the social navigation task as a Partially Observable Markov Decision Process (POMDP), defined by tuple $\mathcal{M}=\left\langle\mathcal{S},\mathcal{A},\mathcal{O},\mathcal{T},\mathcal{R},\Omega,\gamma\right\rangle$. Here, $\mathcal{S}$ represents the state space and $\mathcal{A}$ denotes the action space. The transition function $\mathcal{T}:\mathcal{S}\times\mathcal{A}\times\mathcal{S}\mapsto [0,1]$ defines the environment dynamics, including both the robot's kinematics and the crowd's evolution. $\mathcal{R}:\mathcal{S}\times\mathcal{A}\mapsto\mathbb{R}$ is the reward function.
\par Crucially, the robot does not have access to the full states $s_t$ of other dynamic agents. Instead, it receives an observation $o_t\in\mathcal{O}$ governed by the observation probability $\Omega (o_t|s_t)$. The objective of deep reinforcement learning in this context is to learn a policy $\pi:\mathcal{O}\mapsto\mathcal{A}$ that maximizes the discounted cumulative reward while satisfying collision avoidance and task completion constraints:
\begin{equation}
    \begin{aligned}
        \pi^* &= \arg\max_{\pi} \sum_{t=0}^{T} 
        \mathbb{E}_{\substack{
            o_t \sim \Omega(\cdot \mid s_t),a_t \sim \pi(\cdot \mid o_t) \\
            s_{t+1} \sim \mathcal{T}(\cdot \mid s_t, a_t)
        }} \left[ \gamma^t \, r(s_t, a_t) \right] \\
        \operatorname{s.t.} 
        & \ \mathcal{C}(\tau_{\pi^*})\cap\mathcal{C}(\tau _i)=\emptyset \ (i=1,2\ldots,H) \\
        & \ s_0=s_{start}, \ s_{T+1}=s_{goal}
    \end{aligned}
    \label{eq:objective-problem-formulation}
\end{equation}
here $\tau_{i}$ and $\tau_{\pi^*}$ denotes the trajectory of human $i$ and the trajectory induced by the optimal policy, respectively. $\mathcal{C}(\tau)$ denotes the space occupied by trajectory $\tau$.
\par When deploying the optimized policy $\pi^*$ in the real world, three key sim-to-real challenges arise. First, $\pi^*$ is often trained under idealized simulated dynamics that assume instantaneous tracking of control commands, whereas in practice, actuator latency causes the executed trajectory to deviate from the simulated one. Second, the number of detected humans $H$ is time-varying due to occlusion and sensor limitation, requiring a policy that adapts to crowds of variable densities. Third, human states (positions and velocities) are not directly observable and must be inferred from raw sensor measurements.
\par To address these issues, as summarized in Fig.~\ref{fig:whole-process}, we begin by establishing a theoretical result showing that tracking error decays exponentially as control order $n$ increases, and further propose using second-order actions as input space for differential drive robots to ensure dynamic feasibility. We then introduce an unbiased residual gating block that balance reaction- and memory-based behaviors, and further enable robust handling of variable numbers of humans through a modified residual connection. Finally, we develop a perception-tracking pipeline that provides deployable estimates of human states. The following sections elaborate on each component.


%% file: contents/approach_ilqr.tex
\subsection{Higher-Order Control For Dynamic Feasibility}\label{subsec:second-order-control}
\par For a general robotic system, dynamic feasibility means that a trajectory must satisfy dynamic constraints and actuator limits. Formally, $\forall t\in[1,T]$, $\exists \mathbf{u}\in\mathcal{U}$ such that $\dot{\mathbf{x}}=f(\mathbf{x},\mathbf{u})$. To better quantify this ideology, we introduce the following concept of order-$n$ control.
\begin{definition}\label{def:order-n-control}
Order-$n$ Control refers to the setting where state vector $\mathbf{x}$ contains derivatives up to order $n-1$, control input $\mathbf{u}$ is the $n$-th derivative, and the $n+1$-th derivative is modeled as a Dirac delta impulse:
\begin{equation}
    \begin{aligned}
        \mathbf{x}(t)&=\left[x^{(0)}(t),\ldots,x^{(n-1)}(t)\right]^{\top} \\ 
        \mathbf{u}(t)&=x^{(n)}(t) \\
        x^{(n+1)}(t)&=\delta (t)
    \end{aligned}
\end{equation}
\end{definition}
This formulation is commonly adopted in discrete-time simulators, where the assumption that $x^{(n+1)}(t)$ is a Dirac delta implies that control inputs can be tracked instantaneously. In practice, this makes $\mathbf{u}(t)$ piecewise constant, an unrealistic approximation for real robots whose actuators exhibit non-negligible latency and cannot track abrupt command changes. The following theorem states that this inconsistency can be effectively bridged as control order $n$ increases.
\begin{theorem}\label{theorem:order-n-tracking-error}
Let control input $\mathbf{u}(t)=x^{(n)}(t)$ be piecewise constant with bounded step changes $\left\lVert \Delta \mathbf{u}\right\rVert <\infty$ at sampling period. Then tracking error for $x^{(k)}(t)$ over time step $\Delta t$ is $O\left(\frac{\Delta t^{n-k}}{(n-k)!}\right)$.
\end{theorem}
The proof is presented in Appendix A. Intuitively, since any error in $x^{(n)}$ is integrated $n-k$ times before affecting $x^{(k)}$, its impact shrinks to $O(\Delta t^{n-k}/(n-k)!)$ according to Taylor expansion. Tracking error decreases at exponential speed according to the theorem, while dimension of state space increases linearly, demanding exponentially more data as stated by the \emph{Curse of Dimensionality}. Below we leverage these analysis to derive a practical method to reduce tracking error while satisfying dynamic feasibility.
\par While in principle the method below applies to any robotic systems, we focus on differential drive robot due to its simplicity and widespread use in navigation tasks. For a first-order differential drive robot, state $\mathbf{x}=[x,y,\theta]^{\top}$, control input $\mathbf{u}=[v,\omega]^{\top}$. The kinematics equations are:
\begin{equation}
    \begin{aligned}
    \dot{x}&=v\cos\theta & \dot{y}&=v\sin\theta & \dot{\theta}=\omega
    \end{aligned}
\end{equation}
with the underlying dynamics equations being:
\begin{equation}
    \begin{aligned}
    \dot{v}&=\frac{1}{m}(F_r+F_l) & \dot{\omega}&=\frac{r}{I}(F_r-F_l)
    \end{aligned}
\end{equation}
where $F_r$, $F_l$ are wheel forces, $m$ is robot mass, $I$ is the moment of inertia and $r$ the half-width of wheelbase. The key insight is that actuator limits $F_{min}\leq F_r,F_l\leq F_{max}$ map directly to acceleration constraints:
\begin{equation}
    \begin{aligned}
    a_{min}&\leq \dot{v}\leq a_{max} & \alpha_{min}\leq \dot{\omega}&\leq \alpha_{max}
    \end{aligned}
\end{equation}
In discrete time with sampling period $\Delta t$, dynamically feasible actions must satisfy:
\begin{equation}
    \begin{aligned}
    |v_{k+1} - v_k| &\leq a_{\max} \Delta t & |\omega_{k+1} - \omega_k| &\leq \alpha_{\max} \Delta t
    \end{aligned}
    \label{eq:acceleration-constraints}
\end{equation}
Using second-order actions as input space therefore provide a practical sweet spot: it enforces acceleration constraints (Eq.~\eqref{eq:acceleration-constraints}) naturally, achieves substantially lower tracking error compared to its first-order counterpart according to Theorem~\ref{theorem:order-n-tracking-error}, while remaining computationally tractable for real-time learning and planning. This choice facilitates the derivation of second-order unicycle models~\cite{differential-drive-dynamics}:
\begin{equation}
    \begin{aligned}
        \dot{x}&=v\cos\theta & \dot{y}&=v\sin\theta & \dot{\theta}&=\omega \\
        \dot{v}&=u_{a} & \dot{\omega}&=u_{\alpha}
    \end{aligned}
\end{equation}
where $u_a(\tau)$ and $u_{\alpha}(\tau)$ separately denote linear and angular accelerations at time $\tau$. When accelerations are bounded for all timesteps and further have a limited set of discontinuities, they are Riemann integrable. Integrating over accelerations delivers continuous linear and angular velocities:
\begin{equation}
    v(t) =v_0+ \int_{0}^{t}u_{a}(\tau)  \,d\tau \quad
    \omega (t) =\omega_0+\int_{0}^{t} u_{\alpha}(\tau)  \,d\tau 
    \label{eq:integrating-over-accelerations}
\end{equation}
which provides the foundation for smooth trajectories.
\par In summary, we propose to use linear and angular accelerations $(u_a,u_{\alpha})$, bounded within the feasible region, as the action space:
\begin{equation}
\mathcal{A}=\{(u_a,u_{\alpha})|u_a\in V,u_{\alpha}\in \Omega\}
\label{eq:action-space}
\end{equation}
where $V=[a_{min},a_{max}],\Omega=[\alpha_{min},\alpha_{max}]$. 

\subsection{Imitation Learning With Stochastic iLQR}\label{subsec:imitation-with-ilqr}
\par Higher-order policies are inherently more difficult to train with reinforcement learning due to their increased state dimensionality and longer temporal credit assignment. Rather than relying on intricate training curricula, we initialize the policy via imitation learning. This would require an expert producing the same order of control signals. Trajectory optimization emerges as a suitable candidate, accounting for different model dynamics, while outputting desired order of actions. We adopt a general formulation known as \emph{Maximum Entropy Control}~\cite{maximum-entropy-control}, which augments the standard trajectory optimization objective with an entropy term that encourages exploration and smooth policies:
\begin{equation}
    \begin{aligned}
    \min _{\mathbf{u}_1,\cdots,\mathbf{u}_{T}} \ & \sum_{t=1}^{T} \mathbb{E}_{\pi}
    \left[\ell \left(\mathbf{x}_t, \mathbf{u}_t\right)\right]-\mathcal{H}
    \left(\pi\left(\mathbf{u}_t|\mathbf{x}_t\right)\right) \\
    \operatorname{s.t.} \ & \mathbf{x}_{t}=f\left(\mathbf{x}_{t-1}, \mathbf{u}_{t-1}\right)
    \end{aligned}
    \label{eq:maximum-entropy-control}
\end{equation}
here $\ell$, $f$ denotes cost and dynamics function, $\mathbf{x}_t$, $\mathbf{u}_t$ denotes state and action at time $t$, respectively. We use this notation throughout this subsection, and use subscripts to denote derivatives, such that $\ell_{\mathbf{x}t}$ denotes the gradient of $\ell$ with respect to $\mathbf{x}_t$ and $\ell_{\mathbf{xu}t}$ denotes the gradient of $\ell$ with respect to $[\mathbf{x}_t;\mathbf{u}_t]$.
\par Solver for Eq.~\eqref{eq:maximum-entropy-control} is chosen to be stochastic iLQR, a variant of Differential Dynamic Programming (DDP), which optimizes trajectory under linear-quadratic assumptions:
\begin{equation}
    \begin{aligned}
    \ell \left(\mathbf{x}_t, \mathbf{u}_t\right)&\approx
    \ell (\hat{\mathbf{x}}_t, \hat{\mathbf{u}}_t)+
    \begin{bmatrix}
    \ell_{\mathbf{x}t} \\ \ell_{\mathbf{u}t}
    \end{bmatrix}^{\top}
    \begin{bmatrix}
    \delta \mathbf{x}_t \\
    \delta \mathbf{u}_t
    \end{bmatrix} \\ &+
    \frac{1}{2}
    \begin{bmatrix}
    \delta \mathbf{x}_t \\
    \delta \mathbf{u}_t
    \end{bmatrix}^{\top}
    \begin{bmatrix}
    \ell_{\mathbf{x}\mathbf{x}t} & \ell_{\mathbf{x}\mathbf{u}t} \\
    \ell_{\mathbf{u}\mathbf{x}t} & \ell_{\mathbf{u}\mathbf{u}t}
    \end{bmatrix}
    \begin{bmatrix}
    \delta \mathbf{x}_t \\
    \delta \mathbf{u}_t
    \end{bmatrix} \\
    f(\mathbf{x}_t,\mathbf{u}_t)&\approx f(\hat{\mathbf{x}}_t,\hat{\mathbf{u}}_t) + 
    \begin{bmatrix}
        f_{\mathbf{x}t} \\ f_{\mathbf{u}t}
    \end{bmatrix}^{\top}
    \begin{bmatrix}
    \delta \mathbf{x}_t \\
    \delta \mathbf{u}_t
    \end{bmatrix}
    \end{aligned}
    \label{eq:linear-quadratic-assumptions}
\end{equation}
where $\delta \mathbf{x}_t=\mathbf{x}_t-\hat{\mathbf{x}}_t, \delta \mathbf{u}_t=\mathbf{u}_t-\hat{\mathbf{u}}_t$, and we will denote $f_{ct}\coloneqq f(\hat{\mathbf{x}}_t,\hat{\mathbf{u}}_t)$. Define $Q$ function as the cost-to-go and value function as the expected value of $Q$ function under policy $\pi$:
\begin{equation}
\begin{aligned}
Q(\mathbf{x}_t, \mathbf{u}_t)&=\ell(\mathbf{x}_t,\mathbf{u}_t)+\sum_{t' = t+1}^{T} \mathbb{E}_{\pi} \left[\ell(\mathbf{x}_{t'},\mathbf{u}_{t'})\right] \\
V(\mathbf{x}_t)&=
\mathbb{E}_{\pi}\left[Q(\mathbf{x}_t, \mathbf{u}_t)\right]
\end{aligned}
\end{equation}
Prior work~\cite{guided-policy-search} has established the following solution for maximum entropy objective. For clarity, we formalize the statement here and provide a complete proof in Appendix B.
\begin{theorem}\label{theorem:maximum-entropy-controller}
Under linear dynamics and quadratic cost, solution to problem~\eqref{eq:maximum-entropy-control} is a stochastic Gaussian policy:
\begin{equation}
    \mathbf{u}_t\sim\mathcal{N}(\hat{\mathbf{u}}_t
    +\mathbf{K}_{t}(\mathbf{x}_t-\hat{\mathbf{x}}_t)+\mathbf{k}_t,
    Q_{\mathbf{u},\mathbf{u}t}^{-1})
    \label{eq:stochastic-gaussian}
\end{equation}
where $\mathbf{K}_{t}=-Q_{\mathbf{u,u}t}^{-1}Q_{\mathbf{u,x}t},\mathbf{k}_{t}=-Q_{\mathbf{u,u}t}^{-1}Q_{\mathbf{u}t}$, with Q and value functions satisfying the following recurrence relation:
\begin{equation}
    \begin{aligned}
        \label{eq:standard-ilqr-backward-pass}
    Q_{\mathbf{xu}t}&=\ell_{\mathbf{xu}t}+
    f_{\mathbf{xu}t}^{\top}V_{\mathbf{x}t+1}+ 
    f_{\mathbf{xu}t}^{\top}V_{\mathbf{x},\mathbf{x}t+1}f_{ct} \\
    Q_{\mathbf{xu},\mathbf{xu}t}&=\ell_{\mathbf{xu},\mathbf{xu}t}+f_{\mathbf{xu}t}
    ^{\top}V_{\mathbf{x,x}t+1}f_{\mathbf{xu}t} \\
    V_{\mathbf{x}t}&=Q_{\mathbf{x}t}-Q_{\mathbf{u,x}t}^{\top}Q_{\mathbf{u,u}t}^{-1}
    Q_{\mathbf{u}t} \\
    V_{\mathbf{x,x}t}&=Q_{\mathbf{x,x}t}-Q_{\mathbf{u,x}t}^{\top}Q_{\mathbf{u,u}t}^{-1}
    Q_{\mathbf{u,x}t}
    \end{aligned}
\end{equation}
\end{theorem}
Building on these theoretical results, we now introduce two contributions that enable effective use of stochastic iLQR for imitation learning in social navigation.
\par Firstly, we propose to leverage a divergence minimization objective for imitation learning. While prior works have considered divergence minimization primarily from an interpretive perspective~\cite{imitation-as-f-divergence},~\cite{right-form-of-divergence}, our stochastic iLQR framework allows us to directly implement this principle, initializing policy to be the M-Projection of iLQR output distribution $\mathcal{N}_{ilqr}(\mu,\Sigma)$:
\begin{equation}
\pi=\arg\min _{\pi} D_{KL}\left(\mathcal{N}_{ilqr}(\mu,\Sigma) \Vert \pi\right)
\label{eq:divergence-min-for-imitation}
\end{equation}
where $D_{KL}(\cdot\Vert\cdot)$ is the Kullback-Leibler (KL) divergence. When the target distribution appears in the first argument, as in Eq.~\eqref{eq:divergence-min-for-imitation}, the objective is referred to as forward-KL; when policy $\pi$ appears in the first argument, it is known as reverse-KL. Forward-KL reduces to Mean-Squared-Error (MSE) when covariance $\Sigma$ is the identity matrix. In our setting, covariance is the curvature of $Q$ function, thereby encoding meaningful structure about action uncertainty, which is passed to policy $\pi$ through objective~\eqref{eq:divergence-min-for-imitation}. Minimizing forward-KL yields the M-projection, which is mode-covering, whereas minimizing reverse-KL produces the I-projection, which is mode-seeking. These two projections will be further compared in experiment Section~\ref{subsec:experiment-imitation-reinforcement}.
\par Secondly, we design a cost function tailored for social navigation contexts, which consists of three parts: goal reaching, smooth trajectory and collision avoidance. Denote the weighted square norm of $\mathbf{x}$ as $\left\lVert \mathbf{x}\right\rVert _Q=\mathbf{x}^{\top}Q\mathbf{x}$, overall stage cost is given as:
\begin{equation}
    \begin{aligned}
    \ell(\mathbf{x}_i,\mathbf{u}_i)= &\sum_{i=1}^{N}\biggl[ \left\lVert \mathbf{x}_i-\mathbf{g}_{i} \right\rVert _{Q_g} + \left\lVert \mathbf{u}_i\right\rVert _{Q_u} \\ 
    & + \sum_h \max \left( 0, \left\lVert d_{s} - \left\lVert \mathbf{p}_i-\mathbf{p}_i^h\right\rVert _2 \right\rVert _{Q_h} \right)\biggr]
    \end{aligned}
    \label{eq:ilqr-cost-function}
\end{equation}
where $i$ denotes time step within control horizon $N$, $\mathbf{p}_i$, $\mathbf{p}_i^h$ denotes robot position (which is the first two dimensions of state $\mathbf{x}_i$) and $h$-th human position at step $i$, respectively. $d_{s}$ is a safety threshold and $Q_g$, $Q_u$, $Q_h$ are weight matrices. Goal position $\mathbf{g}_{i}$ is selected as:
\begin{equation}
    \label{eq:moving-goal-position}
    \mathbf{g}_i=[g_x, g_y,
    \arctan \left( \frac{y_i-g_y}{x_i-g_x} \right),
    0,0]^{\top}
\end{equation}
with $(g_x,g_y)$ being the global goal coordinates. Empirically, setting a predictive goal for orientation (instead of using a constant one) mitigated hysteresis problem when controlling second-order differential drive robot. Within the control horizon, we assume constant velocities for humans to predict their future motions.
\par The collision avoidance term in cost function~\eqref{eq:ilqr-cost-function} is non-convex, combined with non-linear differential drive dynamics, this forms a highly non-convex global optimization problem. ILQR manages this complexity by iteratively constructing locally convex quadratic cost and linear dynamics approximations (Eq.~\eqref{eq:linear-quadratic-assumptions}). Standard iLQR optimizes the control sequence via iterative forward and backward passes. Given a candidate solution $\mathbf{u}_{1:N}^{j}$, iLQR rolls out the true dynamics $f$ to obtain a trajectory $\{(\mathbf{x}_i,\mathbf{u}_i)\}_{i=1}^N$. During the backward pass (Eq.~\eqref{eq:standard-ilqr-backward-pass}), gradients and Hessians of the cost $\ell$ and dynamics $f$ are evaluated to recursively compute the $Q$ function derivatives and the next control candidate $\mathbf{u}_{1:N}^{j+1}$. To avoid non-convex collision avoidance cost producing indefinite Hessians, we apply \emph{Hessian regularization} to enforce strictly local convexity and guarantee a valid descent direction. Furthermore, since the linear-quadratic expansion is only accurate within a local vicinity of the reference trajectory, a forward-pass \emph{line search} method from~\cite{iLQR-line-search-regularization} is adopted to evaluate the accuracy of dynamics approximation and prevent overly aggressive updates. These two techniques dynamically establish a trust region: successful line searches decrease regularization for faster, Newton-like steps, whereas repeated failures signal a distorted local model, prompting increased regularization for conservative, gradient-like updates.
\par A common problem in imitation learning is distribution shift, where the student found itself in states unlabeled by the expert. To mitigate this issue, DAgger~\cite{dagger} is employed to further initialize the policy. The iLQR-based imitation learning process is described in Algorithm~\ref{alg:ilqr-imitation-learning}.
\begin{figure}[!t]
    \removelatexerror
    \begin{algorithm}[H]
        \caption{Imitation learning with stochastic iLQR}
        \label{alg:ilqr-imitation-learning}
    \DontPrintSemicolon
    \SetAlgoLined
    \SetKw{init}{Initialize}
    \SetKwFunction{backward}{iLQRBackwardPass}
    \SetKwFunction{forward}{iLQRForwardPass}
    \SetKwFunction{cost}{iLQRCost}
    \SetKwFunction{ilqr}{iLQR}
    \SetKwFunction{step}{envStep}
    \SetKwFunction{clip}{clip}
    \SetKwFunction{search}{lineSearch}
    
    \KwIn{planning horizon $N$, positions $\{\mathbf{p}^{start},\mathbf{p}^{goal}\}$, data gather episodes $n_{ilqr}$, supervised training episodes $\{n_{bc}, n_{dagger}\}$}
    \KwOut{M-Projection of iLQR distribution}
    \init{$\hat{\mathbf{u}}_{1:N}\gets \mathbf{0}, 
    \mathcal{D}_{iLQR}\gets\emptyset,\hat{\mathbf{x}}\gets \mathbf{p}^{start}$}\;
    \For{$j\gets 0$ \KwTo $n_{ilqr}$}{
        $f_{\mathbf{xu}},\ell_{\mathbf{xu}},\ell_{\mathbf{xu,xu}}\gets$
        \forward{$\hat{\mathbf{u}}_{1:N},f,\mathbf{p}^{goal}$}\;
        $\mathbf{k},\mathbf{K}\gets$\backward{$f_{\mathbf{xu}},\ell_{\mathbf{xu}},\ell_{\mathbf{xu,xu}}$}\;
        $\mathbf{u}_{1:N}^*\gets$\search{$\hat{\mathbf{u}}_{1:N}+\mathbf{K}(\mathbf{x}-\hat{\mathbf{x}})+\alpha\mathbf{k}$}\;
        $(\mathbf{x}^*,r^*)\gets$\step{$\mathbf{u}^*_1$}\;
        $\mathcal{D}_{iLQR}\gets\mathcal{D}_{iLQR}\cup (\hat{\mathbf{x}},\mathbf{u}_0^*)$\;
        $\hat{\mathbf{u}}_{1:N}\gets\mathbf{u}_{1:N}^*,\hat{\mathbf{x}}\gets\mathbf{x}^*$
    }
    \For{$n_{episode}\gets 0$ \KwTo $n_{bc}$}{
        $\pi_{bc}\gets\arg\min _{\pi} D_{KL}\left(\mathcal{N}_{ilqr}(\mu,\Sigma) \Vert \pi\right)$\;
    }
    \For{$n_{episode}\gets 0$ \KwTo $n_{dagger}$}{
        $\mathbf{u}_t\sim \pi_{bc}(\cdot\vert\mathbf{x}_t)$\;
        $\mathbf{u}_t^*\gets$\ilqr{$\mathbf{x}_t$}\;
        $\mathbf{x}_{t+1}\gets$\step{$\mathbf{u}_t$}\;
        $\mathcal{D}_{iLQR}\gets\mathcal{D}_{iLQR}\cup (\mathbf{x}_t,\mathbf{u}_t^*)$\;
        $\pi_{bc}\gets\arg\min _{\pi} D_{KL}\left(\mathcal{N}_{ilqr}(\mu,\Sigma) \Vert \pi\right)$\;
    }
    \KwRet{$\pi_{bc}$}
    \end{algorithm}
\end{figure}

%% file: contents/approach_clustering.tex
\subsection{Tracking Pipeline}\label{subsec:clustering}
\begin{figure}[htbp]
    \centering
    \includegraphics[width=1.0\linewidth]{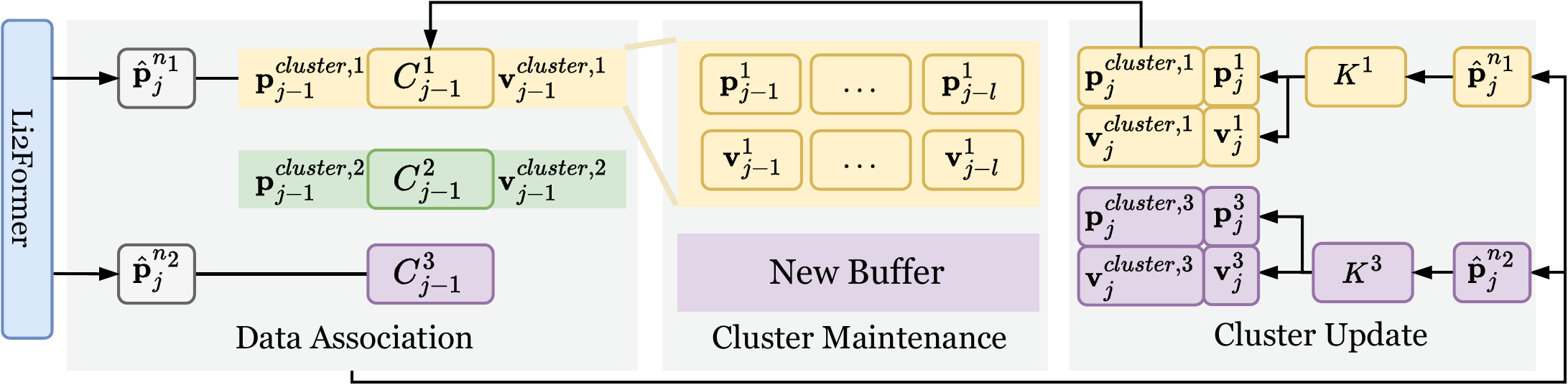}
    \caption{Pictorial illustration of the proposed clustering algorithm. Whole procedure 
    consists of three parts: (1) assign new detections to clusters (data association); (2) maintain a buffer of past positions and velocities for centroid estimation (cluster maintenance); and (3) add denoised position and velocity to cluster buffer (cluster update). In this example, detection $n_1$ is assigned to cluster 1, while $n_2$ initializes a new cluster. The matched detection is then processed by the cluster's Kalman filter, and the resulting denoised state is added to cluster buffer.}
    \label{fig:clustering-procedure}
\end{figure}
While accurate human states such as positions and velocities can be directly acquired in simulation, estimating these quantities in real-world deployment using only a 2D LiDAR is a nontrivial task. Throughout this subsection, we use superscripts for entities and subscripts to indicate time steps, such that $p_{xj}^i$ denotes the $x$-coordinate of $i$-th entity at time step $j$. Robot will be denoted as entity 0. We use vector notation $\mathbf{p}_{j}^i=(p^i_{xj},p^i_{yj})$ and $\mathbf{v}_{j}^i=(v^i_{xj},v^i_{yj})$ to represent the position and velocity of human $i$ at time $j$, respectively. The tracking pipeline is visualized in Fig.~\ref{fig:clustering-procedure}, with key components explained below.
\subsubsection{Cluster Maintenance}
The trajectory of human $i$ over $[t_1,t_{j-1}]$ with step $\Delta t$ consists of tuples $\{(\mathbf{p}_{k}^i,\mathbf{v}_{k}^i)\}_{k=1}^{j-1}$, which we denote as cluster $C^i_{j-1}$. We represent each human by the \textit{centroid} of its associated point cluster, enabling more stable position and velocity estimates via historical averaging. Each cluster maintains a buffer of positions and velocities from the past $l$ time steps. At time step $j-1$, centroid position is the newest buffer entry, while velocity is computed via weighted historical averaging:
\begin{equation}
    \begin{aligned}
        \mathbf{v}^{cluster,i}_{j-1}&=\frac{\sum_{k=j-l}^{j-1}w_j\mathbf{v}_j^i}{\sum_{k=j-l}^{j-1}w_j}\quad
        w_j=\begin{cases}
            \tanh \left(\frac{\lambda}{t_j-t_k}\right), \ & k<j \\
            1, \ & k=j
        \end{cases} \\
        \mathbf{p}^{cluster,i}_{j-1}&=\mathbf{p}^i_{j-1}
    \end{aligned}
\label{eq:weighted-historical-average}
\end{equation}

\subsubsection{Data Association}\label{subsubsec:data-association}
Given 2D LiDAR point clouds, Li2Former~\cite{li2former} is employed for estimating the coordinates of human centers, obtaining $\hat{\mathbf{p}}_{j}^{n_k}$. Traditional metric for assignment cost, such as Euclidean distance between positions, cannot distinguish between humans that are spatially close. We propose to incorporate a velocity-based term to differentiate between humans who are in close proximity but have distinct temporal behaviors. At time step $j$, distance between raw detection $\hat{\mathbf{p}}_j^{n}$ and cluster $C_j^i$ is defined as:
\begin{equation}
    \begin{aligned}
    \mathcal{J}(\hat{\mathbf{p}}_j^{n},C_{j-1}^i)=&(1- \alpha)
    \left\lVert \hat{\mathbf{p}}_j^{n}-\mathbf{p}_{j-1}^{cluster,i}\right\rVert \\ &+ \alpha
    \left\lVert \hat{\mathbf{v}}_j^{n}-\mathbf{v}_{j-1}^{cluster,i}\right\rVert 
    \label{eq:objective-matching}
    \end{aligned}
\end{equation}
where velocity heuristic $\hat{\mathbf{v}}_j^{n}= (\hat{\mathbf{p}}_j^{n}-\mathbf{p}_{j-1}^{cluster,i})/\Delta t$, and $\alpha$ is a weight term determining the relative importance of spatial and temporal similarities. \textit{Hungarian Algorithm} is then applied to the full cost matrix to optimally assign points to clusters. Detection $\mathbf{p}_j^{n_k}$ is initialized as a new cluster if $\mathcal{J}(\hat{\mathbf{p}}_j^{n_k},C_{j-1}^i)>d_{max}$ for all $i$ (e.g, $\hat{\mathbf{p}}_{j}^{n_2}$ in Fig.~\ref{fig:clustering-procedure}). Unassigned clusters (such as $C_{j-1}^2$ in Fig.~\ref{fig:clustering-procedure}) rely on their Kalman filter for a limited number of prediction steps and are removed if they remain unmatched during this period.

\subsubsection{Cluster Update}
After assigning points to clusters at time $j$, corresponding Kalman filter $K^i$ is applied to obtain denoised positions $\mathbf{p}_j^i$ and provide preliminary velocity estimates $\mathbf{v}_j^i$, which are added in the buffer for centroid state estimation in the next time step.

\subsubsection{State Space}
For robot, at time step $j$, joint state $\mathbf{s}^0_j$ encompasses two parts: its fully observable state $s^0_j$ and humans' partially observable states $o^i_j$, given by: 
\begin{equation}
    \begin{aligned}
    s^0_j&=[r^0,p_{xj}^0,p_{yj}^0,\theta^0_j,g_x^0,g_y^0,v_j^0,\omega_j^0] \\
    o^i_j&=[r^i,p_{xj}^i,p_{yj}^i,v_{xj}^i,v_{yj}^i], \ i=1,\ldots,H
    \end{aligned}
    \label{eq:state-parameterization}
\end{equation}
where $r^0$ is robot radius, $\theta_j^0$ denotes robot orientation, $(g_x^0,g_y^0)$ is the preset constant goal position and $(v_j^0,\omega_j^0)$ represents linear and angular velocities for robot at time $j$. Given $H$ humans, the observation for robot at time step $j$ constitutes of $\left[s^0_j,o^1_j,o^2_j,\ldots,o^{H}_j \right]$. 

%% file: contents/approach_net.tex
\subsection{Gated Spatio-Temporal Transformer}\label{subsec:network-arch}
\begin{figure*}[!ht]
    \centering
    \includegraphics[width=1.0\textwidth]{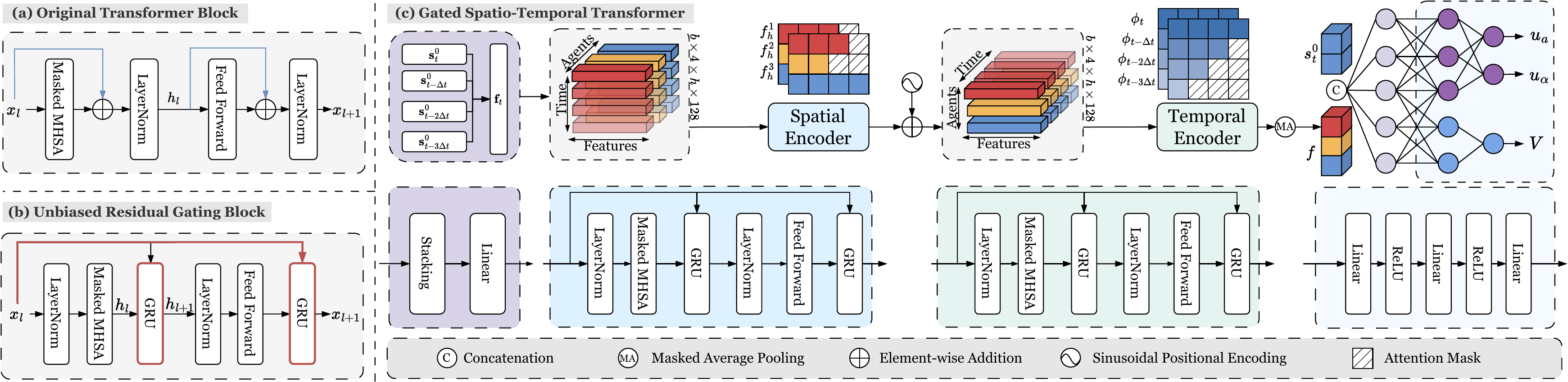}
    \caption{The proposed unbiased residual gating block and the gated spatio-temporal transformer architecture that builds on top of it. (a) Original transformer block, with reaction pathway marked in blue. Key limitations include non-Markovian initialization due to post-LayerNorm and imbalance between reaction and memory pathways due to simple additive residual. (b) Proposed gated residual block that replaces the additive residual with GRU to regulate  reaction and memory pathway. GRU bias terms are removed and skip connection is modified to handle variable-length inputs with padding. Modifications are highlighted in red. (c) Complete gated transformer architecture. Joint states are passed into spatial encoder, which encodes attention in the crowd. Spatial features are then reshaped and passed to temporal encoder to extract relationships across time dimension. Output features are passed into subsequent actor and value networks.}
    \label{fig:network-arch}
\end{figure*}
\par Social navigation requires the robot to balance two types of behaviors: \textit{reactive reflexes} (e.g., stopping immediately to avoid colliding a human) and \textit{contemplative planning} (e.g., remembering the movement of a human to pass elegantly). The transformer architecture~\cite{transformer} naturally support the emergence of these behaviors. The residual stream (blue arrows in Fig.~\ref{fig:network-arch}(a)) propagates the embedding of the current observation $o_t$ through the depth of the network. This can be interpreted as the \textit{reaction pathway}. Conversely, the attention stream (black arrows in Fig.~\ref{fig:network-arch}(a)) processes history observations $o_{t:t-k}$ via Multi-Head Attention (MHA) mechanism, injecting historical context into the state representation. This can be seen as the \textit{memory pathway}. In this subsection, we first analyze why the original transformer architecture is hard to train in RL setting, then state prior attempts to alleviate these issues, followed by our unique contributions in both architectural design and training methodology.
\par Standard transformer block (Fig.~\ref{fig:network-arch}(a)) suffer from poor initialization in RL because the Layer Normalization (LN) operation mixes the above two pathways immediately:
\begin{equation}
h_{l}=\operatorname{LN}(x_{l}+\operatorname{MHA}(x_{l}))
\label{eq:forward-original-transformer}
\end{equation}
Upon initialization, $\operatorname{MHA}(x_{l})\approx 0$, and Eq.~\eqref{eq:forward-original-transformer} evaluates to $\operatorname{LN}(x_{l})\neq x_{l}$, preventing the agent from starting with a stable, Markovian policy. Moreover, contextual perturbation directly adds to reaction pathway, yet the agent must weight these two sources of behavior appropriately rather than combine them in a fixed manner.
\par To address the problem of Markovian initialization, prior works~\cite{gru-transformers} proposed placing LN inside the residual branch:
\begin{equation}
h_{l}=x_{l}+\operatorname{MHA}(\operatorname{LN}(x_{l}))
\end{equation}
which ensures $h_{l}\approx x_{l}$ upon initialization, thus enabling the policy to begin as a Markovian mapping. To mitigate reaction-memory imbalance,~\cite{gru-transformers} propose replacing additive residual connections with various gating mechanisms, among which GRU~\cite{gated-recurrent-unit} achieves the best performance. Forward equations for GRU following the MHA block (Fig.~\ref{fig:network-arch}(b)) are given by:
\begin{equation}
    \begin{aligned}
    \Gamma _t^u &= \sigma\left( W_{u}h_{l}+U_{u}x_{l}+b_{u}\right) \\
    \Gamma _t^r &= \sigma\left( W_{r}h_{l}+U_{r}x_{l}\right) \\
    \tilde{h}_{l+1} &= \tanh \left(W_{h}h_{l}+U_{h}\left(\Gamma _t^r\odot x_{l}\right)\right) \\
    h_{l+1}&=\Gamma _t^u \odot \tilde{h}_{l+1} + (1-\Gamma _t^u)\odot x_{l}
    \end{aligned}
    \label{eq:gru-forward-equation}
\end{equation}
where $\odot$ denotes Hadamard product. The mechanism is that network learns to output $\Gamma _t^u\approx 0$ when the memory pathway is noisy (favor reaction pathway), and $\Gamma _t^u\approx 1$ when memory contains critical information. In order to preserve Markovian property, $b_u$ in Eq.~\eqref{eq:gru-forward-equation} is often initialized as a small negative value, thus saturating the sigmoid function to produce $\Gamma _t^u\approx 0$ upon initialization.
\par A core challenge in real-world social navigation is the stochasticity of crowd density due to both scenarios and sensing limitations. To handle this variability, we propose \textit{Unbiased Residual Gating}, an extension of the mechanisms discussed above, featuring two key modifications to the GRU bias terms and the residual pathway, illustrated in Fig.~\ref{fig:network-arch}(b).
\par Although transformer architectures can in principle handle variable-length inputs, training exclusively on fixed-cardinality creates a significant distribution shift between training and deployment, which we later verify in experiment Section~\ref{subsec:net-arch-evaluation}. To bridge this gap, we adopt a variable-cardinality training regime where the number of agents varies dynamically per episode. This necessitates the use of padding for batched tensor operations. In order to avoid adding into padded entries, we remove bias term $b_u$ from Eq.~\eqref{eq:gru-forward-equation}. While this sounds trivial, removing $b_u$ indicates that we cannot enforce $\Gamma_{t}^u\approx 0$ at initialization. Under standard small-weight initialization, unbiased GRU simplifies to:
\begin{equation}
    \begin{aligned}
        \Gamma _t^u &\approx 0.5 & \tilde{h}_{l+1} &\approx 0 & h_{l+1}& \approx 0.5x_{l}
    \end{aligned}
\label{eq:gru-value-init}
\end{equation}
so although the direction of $x_{l}$ is preserved, the signal is scaled by $0.5$. With the original residual structure, this results in:
\begin{equation}
x_{l+1}\approx 0.5h_{l}\approx 0.25x_{l}
\end{equation}
To mitigate this significant cumulative compression, we modify the residual connection so that the second GRU block receives its skip connection directly from input embedding $x_{l}$, rather than from the intermediate state $h_{l+1}$. This produces $x_{l+1}\approx 0.5x_{l}$ at initialization. Over $n$ layers, this reduces information attentuation upon initialization by a factor of $2^n$. Later experiments in Section~\ref{subsec:net-arch-evaluation} will demonstrate that this improvement in information preservation translates into substantial gain in navigation success rate.
\par The proposed gated spatio-temporal transformer architecture (Fig.~\ref{fig:network-arch}(c)) builds upon unbiased residual gating block, using spatial and temporal encoder to model human-robot and human-human interactions on both spatial and temporal dimensions. A critical observation in our experiments is that the proposed gated transformer policy is highly sensitive to the diversity of training scenarios. When trained solely on symmetric settings, such as environments where all humans move toward antipodal points on a circle, the policy constantly converges to a passive strategy: the robot simply waits for the crowd to clear before moving towards the goal. This stems from the expressivity of transformer architecture: in low-entropy environments, the attention mechanism lacks sufficient gradient signal to distinguish between humans, thus treated the crowd as a single entity.
\par To address this, we introduce \textit{Scenario-Based Domain Randomization}, which increases state-space entropy by enriching human motion patterns. Instead of using a symmetric template, we augment the training distribution by varying human start-goal configurations and randomizing human velocities. Under these more varied scenarios, the same passive strategy becomes risky, encouraging the policy to develop genuinely proactive behaviors. The central insight is that sufficient variability is not only a robustness tool, as in conventional domain randomization~\cite{domain-randomization}, but a necessary condition for attention-based policies to learn active and cooperative navigation behaviors. The exact set of effective variations and their influence on policy learning are detailed in experiment Section~\ref{subsec:net-arch-evaluation}.

%% file: contents/approach_drl.tex
\subsection{Deep Reinforcement Learning}
Reward function and deep reinforcement learning algorithm will be described in detail in this section, refer to Section~\ref{subsec:clustering} for state space formulation and Section~\ref{subsec:second-order-control} for the choice of action space.
\subsubsection{Reward Function}\label{sec:reward-function}
Apart from being collision-free, the navigation policy must also respect certain implicit social norms. This work focuses specifically on maintaining an appropriate interpersonal distance, which serves as a practical geometric surrogate for human comfort. We model the robot and each human as circular objects with radii $r$ and $r^h$, and denote the position of their centers at time $t$ by $\mathbf{p}_t$ and $\mathbf{p}^h_t$. Let $r^s$ be the minimum comfort distance required by humans. The resulting social-distance requirement is:
\begin{equation}
\left\lVert \mathbf{p}_{t}-\mathbf{p}_{t}^h \right\rVert\geq  r+r^h+r^s, \ \forall h, \ \forall t
\end{equation}
Denote $d_{min}\coloneqq r+r^h+r^s$, and let $d_t\coloneqq\min_h\left\lVert\mathbf{p}_t-\mathbf{p}_t^{h} \right\rVert$ be distance to the nearest human. This constraint is incorporated into the reward as a soft distance penalty, used alongside a goal-reaching bonus, a collision penalty, and a small time cost. Together, these components encourage the policy to complete the task efficiently while avoiding violations of personal space:
\begin{equation}
    r(s_t,a_t)=\begin{cases}
        1.0, &\text{if }d_t^g <r \\
        -0.25, &\text{else if }d_t < r+r^h \\
        0.5(d_t-d_{min})\Delta t, &\text{else if }d_t<d_{min} \\
        -0.01, &\text{otherwise}
    \end{cases}
\end{equation}
where $d_t^g=\left\lVert\mathbf{p}_t-\mathbf{g}\right\rVert$ is the distance to goal at time $t$ and $\Delta t$ is the simulation time step.
\subsubsection{Deep Reinforcement Learning Algorithm}
Proximal Policy Optimization (PPO)~\cite{ppo} with clipped surrogate objective is employed to solve for Eq.~\eqref{eq:objective-problem-formulation}. The surrogate objective in PPO addresses performance collapse with importance ratio clipping, and has the benefit of reusing off-policy data. The resulting optimization target is:
\begin{multline}
J^{CLIP}(\theta^{'})=\mathbb{E}_{t}\biggl[
    \min\biggl( \frac{\pi_{\theta^{'}}(a_t|s_t)}{\pi_{\theta}(a_t|s_t)}\hat{A_t}, \\
\operatorname{clip}\left(\frac{\pi_{\theta^{'}}(a_t|s_t)}{\pi_{\theta}(a_t|s_t)}, 1-\epsilon, 1+\epsilon \right)
\hat{A_t} \biggr) \biggr]
\label{eq:ppo-objective}
\end{multline}
here $\epsilon$ is a value defining the clipping neighborhood, $\hat{A}_t$ is the estimated advantage function at time $t$, obtained via truncated Generalized Advantage Estimation (GAE)~\cite{generalized-advantage-estimation}. $\theta^{'}$, $\theta$ are new and old policy parameters, respectively. Policy $\pi_{\theta}(a_{t}|s_{t})$ is chosen to be a normal distribution over actions.

%% file: contents/experiments.tex
\renewcommand\arraystretch{1.0}
\begin{table}[!ht]
    \centering
    \caption{Collection and description of baseline models}
    \newcolumntype{C}{>{\centering\arraybackslash}X}
    \newcolumntype{L}[1]{>{\centering\arraybackslash}p{#1}}
    \begin{tabularx}{\linewidth}{L{2.5cm}|C|C|C}\hline
    Methods & Order of Model & Differential Drive & DRL \\ \hline
    ORCA~\cite{orca}       & 1 &              &    \\ \hline
    NH-ORCA~\cite{nh-orca} & 1 & $\checkmark$ &    \\ \hline
    MPC~\cite{rl-subgoal-for-mpc} & 2 & $\checkmark$    &    \\ \hline
    SARL~\cite{sarl}    & 1 &                 & $\checkmark$   \\ \hline
    SARL (second-order) & 2 &                 & $\checkmark$   \\ \hline
    Feasible SARL       & 2 & $\checkmark$    & $\checkmark$   \\ \hline
    KinematicRL (Ours)  & 2 & $\checkmark$ & $\checkmark$ \\ \hline
    \end{tabularx}
    \label{tab:baseline-models}
\end{table}
\section{Experiments}\label{sec:experiments}
\renewcommand\arraystretch{1.0}
\begin{table*}[!htb]
    \centering
    \caption{Policy performance in highly dynamic simulation setting with 5 humans} \label{tab:model-performance-settings}
    \begin{tabular}{cccccc|cccc}\hline
    Methods   & $SR$ & $CR$ & $\mu T$ & $DF$ & $Dist.$ & $v_{of}$ & $w_{of}$ & $\Delta a/\Delta t$ & $\kappa$ \\ \hline
    ORCA~\cite{orca}     & 0.58 & 0.42 & \textbf{10.03}$\pm$1.31 & 0.25$\pm$0.18 & 0.003$\pm$0.034 & 0.17 & 0.22 & 0.56$\pm$1.16 & 0.26$\pm$1.22 \\
    iLQR (first-order)   & 0.47 & 0.50 & 10.58$\pm$5.61 & 0.06$\pm$0.09 & 0.073$\pm$0.056 & 0.07 & 0.01 & 0.48$\pm$1.44 & \textbf{0.03}$\pm$0.34 \\
    \cellcolor[HTML]{EFEFEF}iLQR (second-order)  
        & \cellcolor[HTML]{EFEFEF}\textbf{0.59} 
        & \cellcolor[HTML]{EFEFEF}\textbf{0.39} 
        & \cellcolor[HTML]{EFEFEF}11.83$\pm$4.77 
        & \cellcolor[HTML]{EFEFEF}\textbf{0.03}$\pm$0.06 
        & \cellcolor[HTML]{EFEFEF}\textbf{0.088}$\pm$0.055 
        & \cellcolor[HTML]{EFEFEF}\textbf{0.00} 
        & \cellcolor[HTML]{EFEFEF}\textbf{0.00} 
        & \cellcolor[HTML]{EFEFEF}\textbf{0.04}$\pm$0.08 
        & \cellcolor[HTML]{EFEFEF}0.04$\pm$0.73 \\ \hline
    SARL~\cite{sarl}     & \textbf{0.99} & \textbf{0.01} & \textbf{9.44}$\pm$0.88  & 0.10$\pm$0.10 & \textbf{0.130}$\pm$0.055 & 0.10 & 0.53 & 1.24$\pm$3.52 & 1.00$\pm$2.01 \\
    SARL (second-order)  & 0.52 & 0.48 & 11.25$\pm$1.36 & 0.07$\pm$0.08 & 0.082$\pm$0.057 & 0.06 & 0.44 & 0.60$\pm$2.40 & 0.75$\pm$3.13 \\
    Feasible SARL        & 0.77 & 0.22 & 10.52$\pm$3.59 & 0.05$\pm$0.08 & 0.089$\pm$0.056 & 0.00 & 0.00 & 0.21$\pm$0.27 & 0.45$\pm$1.05 \\
    \cellcolor[HTML]{EFEFEF}KinematicRL (Ours)                
        & \cellcolor[HTML]{EFEFEF}0.84 
        & \cellcolor[HTML]{EFEFEF}0.15 
        & \cellcolor[HTML]{EFEFEF}9.76$\pm$2.96 
        & \cellcolor[HTML]{EFEFEF}\textbf{0.03}$\pm$0.06 
        & \cellcolor[HTML]{EFEFEF}0.098$\pm$0.055 
        & \cellcolor[HTML]{EFEFEF}\textbf{0.00} 
        & \cellcolor[HTML]{EFEFEF}\textbf{0.00} 
        & \cellcolor[HTML]{EFEFEF}\textbf{0.16}$\pm$0.44 
        & \cellcolor[HTML]{EFEFEF}\textbf{0.43}$\pm$0.69 \\ \hline
    \end{tabular}
\end{table*}
\begin{figure}[!htb]
    \centering
    \includegraphics[width=1.0\linewidth]{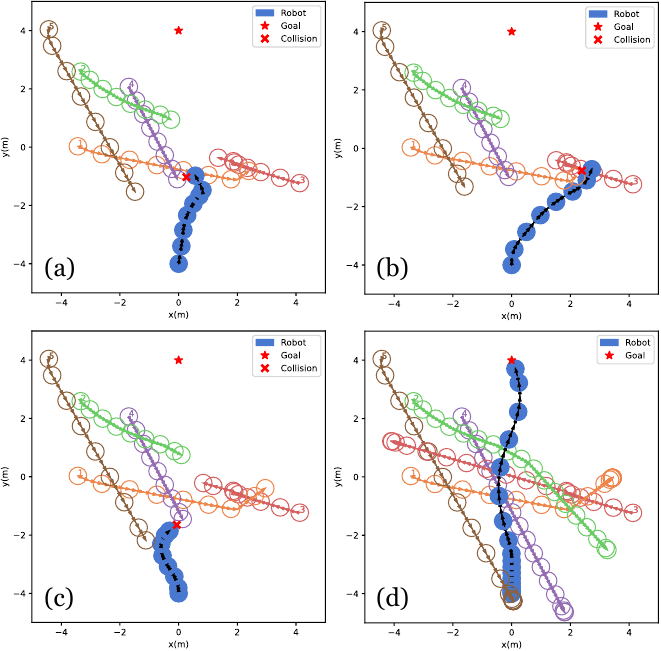} 
    \caption{Comparison of expert trajectories. Among the four expert policies, only the proposed iLQR successfully solves this scenario. Collision points for the other methods are indicated by red crosses. (a) ORCA. (b) NH-ORCA. (c) MPC. (d) iLQR (second-order).}
    \label{fig:expert-policies-qualitative}
\end{figure}
\subsection{Simulation Implementation}
All policies are implemented in Python using PyTorch. We adopted PPO implementation from stable-baselines3~\cite{stable-baselines3} and utilized vectorized environment along with vectorized iLQR to speed-up training, which took approximately 4 hours on a GeForce GTX 1070 Ti GPU and an 8-core Intel Core i7-7700@3.6 GHz CPU. To facilitate reproducibility and future research, we open-source our implementation at \href{https://github.com/Mr-Wonderfool/KinematicRL}{\textcolor{blue}{https://github.com/Mr-Wonderfool/KinematicRL}}.
\subsubsection{Parameter Choices}
The network is trained with learning rate 0.01 after gathering iLQR experiences for 3000 episodes. Adam~\cite{adam} optimizer is used throughout training, learning rate linearly decays from $3\times 10^{-4}$ to $1\times 10^{-6}$ in reinforcement learning phase. Robot's maximum linear and angular accelerations, linear and angular velocities are set to $0.3\, m/s^2$, $0.9\, rad/s^2$, $1\, m/s$ and $\pi\, rad/s$, respectively. Rollout buffer size, GAE lambda and clipping range for PPO are chosen to be 1536, 0.5 and 0.2, respectively. Other parameter choices can be consulted in the code.
\subsubsection{Simulation Settings}\label{sec:simulation-setting}
As elucidated in Section~\ref{subsec:network-arch}, the gated transformer architecture requires sufficiently diverse human motion patterns to learn proactive behaviors. To induce this diversity, we construct a simulation setting with high state-space entropy. Specifically, 30\% of pedestrians are initialized on a circle with radius 4m (adding random perturbations) and assigned antipodal goal locations, while the remaining pedestrians start from uniformly sampled positions within a $10\times 10$ square. Pedestrian velocities are randomized to further increase trajectory variability. The robot is initialized at $(0,-4)$ with goal $(0,4)$ and is invisible to pedestrians, making all humans noncooperative. Following~\cite{sarl}, pedestrians are controlled by ORCA~\cite{orca}, and their preferred velocities are sampled within $[0.5m/s, 1.5m/s]$. We refer to this configuration as the \textit{Highly Dynamic} scenario, characterized by its heterogeneous pedestrian motions and noncooperating agents.
\subsubsection{Baseline Models}
Common expert policies \textit{ORCA}~\cite{orca} and \textit{NH-ORCA}~\cite{nh-orca} are chosen to compare with iLQR. Additionally, \textit{MPC}~\cite{rl-subgoal-for-mpc} is selected as a second-order expert to offer a more direct comparison. The proposed expert policy will be denoted as iLQR (first-order) and iLQR (second-order) based on the order of model used in simulation. For learning-based methods, \textit{SARL}~\cite{sarl} serves as a baseline for attention modeling and kinematic performance. We apply trained SARL to control a robot with acceleration constraints only in test settings, and denote this model as \textit{SARL (second-order)}. To validate the effectiveness of our network architecture, \textit{Feasible SARL} is created by replacing the transformer feature extractor with SARL's value network. For a more holistic view of baseline models, refer to Table~\ref{tab:baseline-models}. Since not all baseline models are designed to handle varying number of humans within batches, to offer a fair comparison, the highly dynamic training environment contains fixed number of humans. Trained models are tested in the same setting for 500 episodes, and result will be denoted as $\mu\pm\sigma$ with $\mu$ and $\sigma$ being mean and standard deviation calculated from all test cases.
\subsubsection{Comparison Metrics}
\par The following metrics are used to evaluate both navigation and kinematic performance of models. 
$SR$: success rate; 
$CR$: collision rate; 
$\mu T$: average navigation time; 
$DF$: frequency of invading humans' comfort distances; 
$Dist.$: average minimum separating distance between robot and humans; 
$v_{of}$: oscillating frequency of linear velocity; 
$\omega_{of}$: oscillating frequency of angular velocity; 
$\Delta a/\Delta t$: acceleration change rate with respect to time (jerk); 
$\kappa$: average curvature of trajectory. 
More specifically, $v_{of}$ is calculated via: 
\begin{equation}
v_{of}=\frac{1}{N}\sum_{i=1}^{N}\sum_{t=0}^{T_{i}-1}
\frac{\mathbbold{1}\{|v^i_{t+1}-v^i_{t}|>a_{max}\Delta t\}}{T_i}
\end{equation}
where $\mathbbold{1}\{\cdot\}$ is the indicator function and $v_{t},v_{t+1}$ are model output velocities at time step $t$ and $t+1$, in evaluation episode $i$. $\omega_{of}$ is calculated with the same rule. Threshold for oscillation is chosen to be the maximum allowed $\Delta v$ and $\Delta \omega$, therefore a positive $v_{of}$ or $\omega_{of}$ value indicates violation of acceleration constraints. $Dist.$ is obtained by calculating the minimum distance to humans in each evaluation case, then averaging over all cases.

\subsection{Imitation-Reinforcement Learning Evaluation}\label{subsec:experiment-imitation-reinforcement}
The proposed workflow consists of imitation learning with iLQR followed by reinforcement 
learning with PPO. Evaluation results for the two phases are presented below.
\subsubsection{Expert Policies}
\par Quantitative results for different expert policies in highly dynamic scenarios are shown in Table~\ref{tab:model-performance-settings}. The proposed iLQR (second-order) outperforms ORCA on all metrics except average navigation time. This is largely because iLQR assumes no reciprocity and explicitly encodes social distance in its objective. Two insights emerge from the results: (1) our method employs second-order control, and respects dynamic feasibility by design. By contrast, ORCA outputs dynamically infeasible actions in nearly 20\% of the cases. We find that most violations occur during sharp avoidance maneuvers, which helps explain its frequent incursions into human comfort zones (25\% of the cases) and its correspondingly low average minimum separation distance (0.003 $Dist.$); (2) the commendable kinematic performance is not attributable solely to the higher-order control formulation but primarily to the model-based structure of trajectory optimization. Ablation between iLQR (first-order) and iLQR (second-order) shows that even with first-order control inputs, frequency of violating acceleration constraints remain low and kinematic performance still surpasses ORCA.
\par Trajectory comparison between four common expert policies are shown in Fig.~\ref{fig:expert-policies-qualitative}. ORCA and NH-ORCA both fall short when encountering agents that broke the reciporcal assumption. Despite formulating collision avoidance as hard constraints, the MPC baseline~\cite{rl-subgoal-for-mpc} still experiences collisions with pedestrians. This failure stems from two compounding factors: (1) the constant-velocity assumption used for predicting human motion within the planning horizon frequently deviates from actual pedestrian trajectories; and (2) the reliance on general-purpose non-linear programming solvers scales poorly with the prediction horizon, forcing the system to execute suboptimal plans calculated within a strict computational budget. In contrast, the proposed iLQR approach successfully navigates to the goal with smooth motions. This stems from a time-varying linear feedback gain alongside the feedforward trajectory, providing immediate stabilization against unmodeled dynamic disturbances. Nevertheless, the system remains subject to the underlying prediction limitations; due to deviations in actual pedestrian motion, the resulting avoidance maneuver still passes closely by human 4 (Fig.~\ref{fig:expert-policies-qualitative}(d), purple circle).

\begin{figure}[!htb]
\centering
    \includegraphics[width=0.8\linewidth]{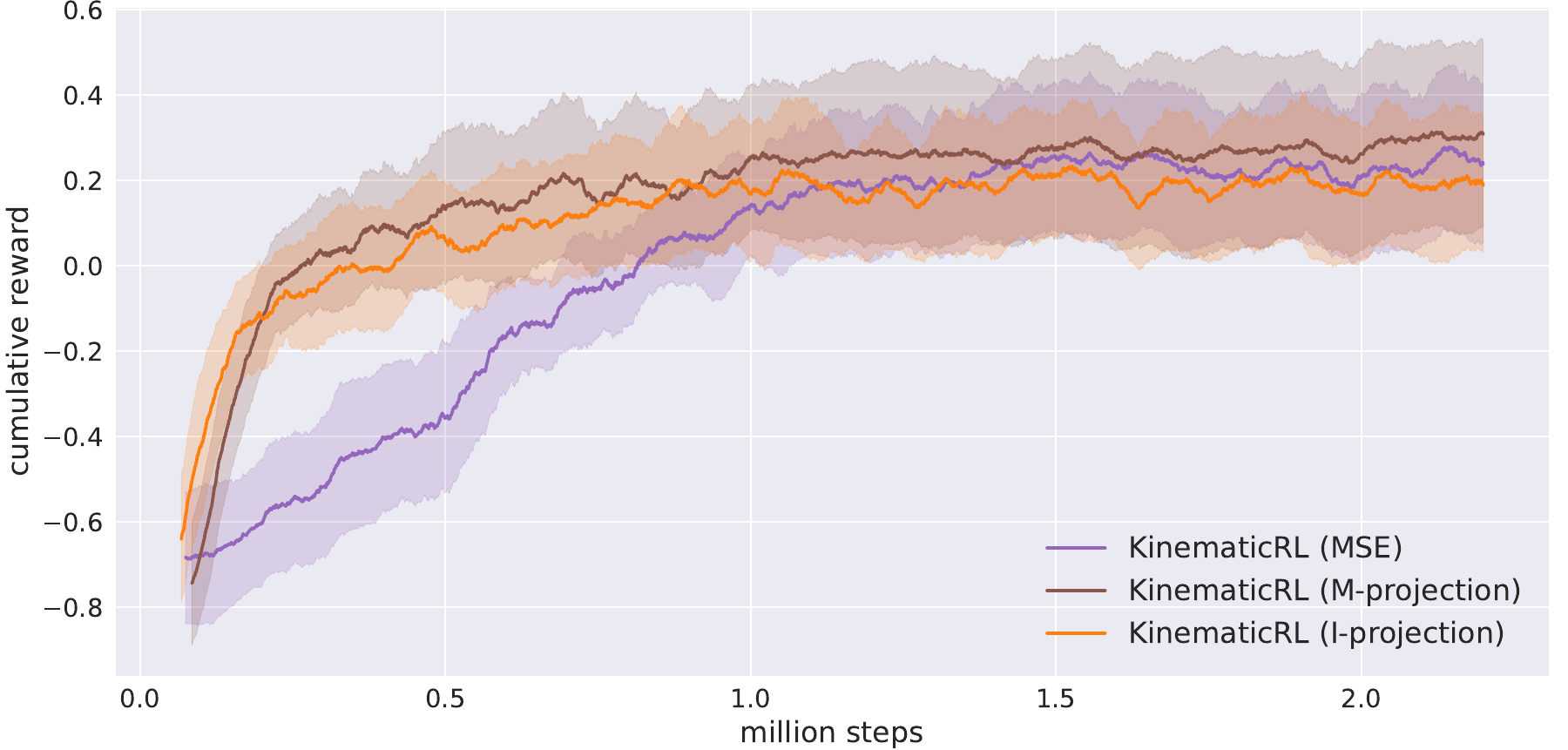}
    \caption{Reward curves during reinforcement learning phase (200-step moving average) for KinematicRL initialized with different imitation learning objectives. KinematicRL with divergence minimization objective converges approximately twice as fast as KinematicRL (MSE), and KinematicRL (M-projection) achieves higher reward compared to KinematicRL (MSE).}
    \label{fig:reward-curve}
\end{figure}
\subsubsection{Imitation Learning}
We propose to use a divergence minimization objective (Eq.~\eqref{eq:divergence-min-for-imitation}) for imitation learning and evaluate how different imitation objectives influence subsequent reinforcement learning efficiency and performance. Fig.~\ref{fig:reward-curve} reports the reward curves. Key observations are: (1) M-projection outperforms I-projection in our setting, which we attribute to the unimodal nature of iLQR demonstrations, making a mode-covering distribution better aligned with the expert's covariance structure; (2) regardless of whether M-projection or I-projection is used, the divergence minimization objective yields reward convergence that is approximately twice as fast as the MSE objective. These results highlight the importance of learning action uncertainty during imitation learning, as an uncertainty-aware policy can better explore high-reward regions during subsequent reinforcement learning.

\subsubsection{Reinforcement Learning}
\begin{figure}[!htb]
    \centering
    \includegraphics[width=1.0\linewidth]{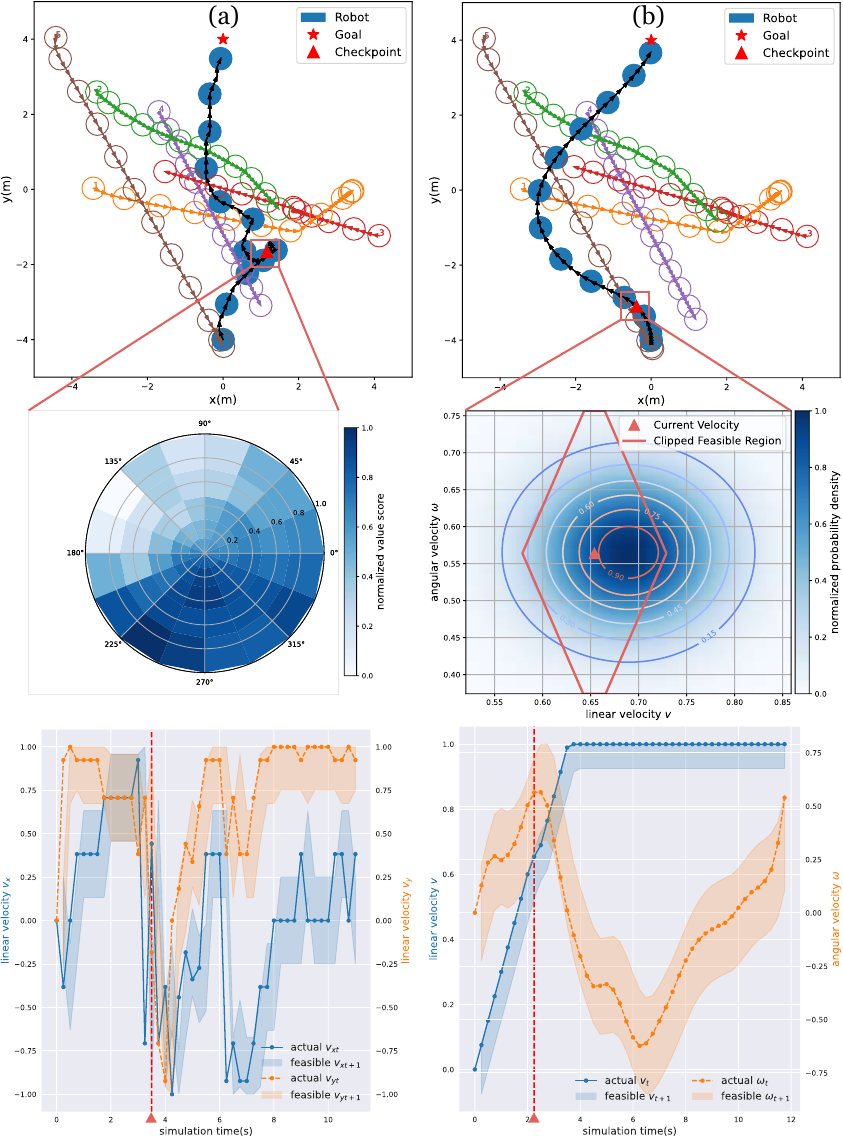}
    \caption{Qualitative comparison in highly dynamic environment. From top to bottom, each column shows the planned trajectory, the value estimates or action distribution, and the velocity curve throughout the episode. Action produced at the red triangle is visualized: for SARL we plot the value estimates over actions $(v,\theta)$ in polar coordinates, and for KinematicRL we plot the Gaussian distribution over velocities (darker color indicates more probable action). The corresponding velocity profiles are shown with the feasible region indicated by a transparent band, and the red vertical line marks the decision point (red triangle). Overall, KinematicRL exhibits safer collision avoidance maneuvers and smoother motions. (a) SARL. (b) KinematicRL.}
    \label{fig:qualitative-comparison-path-vw-dist}
\end{figure}
Results for navigation performance of KinematicRL with baseline models are presented 
in Table~\ref{tab:model-performance-settings}. Our model excels in all kinematic metrics, 
and further analysis of the results consists of three aspects. 
\paragraph{Sim-to-Real Gap}
Performance collapse of SARL (second-order) compared with SARL offer a straight-forward demonstration of sim-to-real gap, evidenced by a significant drop in success rate (47\% in highly dynamic environment) during test phase, indicating that SARL passes almost half of the test cases by executing dynamically infeasible actions. Empirically, these infeasible maneuvers often allow SARL to evade imminent collisions in simulation. Since the simulator models the robot as a first-order system while real robots do not behave as such, this modeling error leads to substantial degradation when dynamics are corrected. This indicates that SARL is not only brittle under realistic dynamics but also potentially unsafe when deployed in real-world settings.
\paragraph{Navigation Performance}
Our model excels at all kinematic metrics compared with SARL, with small jerk and average curvature indicating smooth and comfortable path. This results from imitating model-based expert to inherit the kinematic performance of iLQR, along with second-order control framework. Given that KinematicRL controls a differential drive robot that cannot move sideways and strictly obeys acceleration constraints, a drop in success rate and minimum separating distance is expected. Furthermore, low danger frequency (0.03 in highly dynamic settings) compared with SARL suggests that the proposed model chooses safer paths through better anticipation of human motions, as compared with SARL (0.10 in highly dynamic settings). Although it may seem counterintuitive that collision rate increases while danger frequency decreases, SARL's first-order holonomic model can perform aggressive, last-moment evasive maneuvers that avoid collisions but result in longer durations in close proximity to humans. In contrast, KinematicRL controls a second-order nonholonomic model that executes either a safe pass or no pass, in either scenarios danger frequency is lower compared to SARL.
\paragraph{Collision Avoidance Maneuvers}
Fig.~\ref{fig:qualitative-comparison-path-vw-dist} presents qualitative comparison to justify the above analysis in collision avoidance maneuvers and action smoothness between SARL and KinematicRL. We convert output distribution over accelerations to a distribution over velocities, and plotted the feasible region (Eq.~\eqref{eq:acceleration-constraints}) given the current velocity $(v,\omega)$ (red rhombus in Fig.~\ref{fig:qualitative-comparison-path-vw-dist}(b)). Agents' positions are plotted along with their orientations. SARL exhibits drastic changes in orientation and overlaying positions when avoiding humans (the red triangle on Fig.~\ref{fig:qualitative-comparison-path-vw-dist}(a)). Also, SARL passes in front of human 3, indicating insufficient capture of human intentions. In contrast, our model anticipates motion patterns of human 4 and 5 (purple and brown circles in Fig.~\ref{fig:qualitative-comparison-path-vw-dist}(b)) in the early stage, and opts to rotate towards safer regions. This improved modeling shows self-attention surpasses global attention mechanism when all humans are non-cooperative and require differentiated treatment. Distribution plot and value estimates at checkpoint shows that SARL put high value estimates on aggressive behavior. Comparatively, our model outputs a distribution that largely lies within the feasible region, gradually increases linear velocity for faster goal arrival while preserving a suitable angular velocity.
\paragraph{Feasibility and Smoothness of Actions}
We further plot velocities throughout this episode (Fig.~\ref{fig:qualitative-comparison-path-vw-dist}), and use transparent region to denote feasible action set, so velocities outside this region indicate dynamic infeasibility. As can be inferred from SARL trajectory, the output actions exhibit substantial oscillation. Importantly, SARL outputs a dynamically infeasible action at checkpoint (red vertical line) when avoiding human 3. In contrast, KinematicRL inherently generates feasible actions that are significantly smoother than those of SARL.

\subsection{Network Evaluation}\label{subsec:net-arch-evaluation}
\subsubsection{Generalization to Variable Crowd Densities}
\begin{figure}[!ht]
    \centering
    \includegraphics[width=0.8\linewidth]{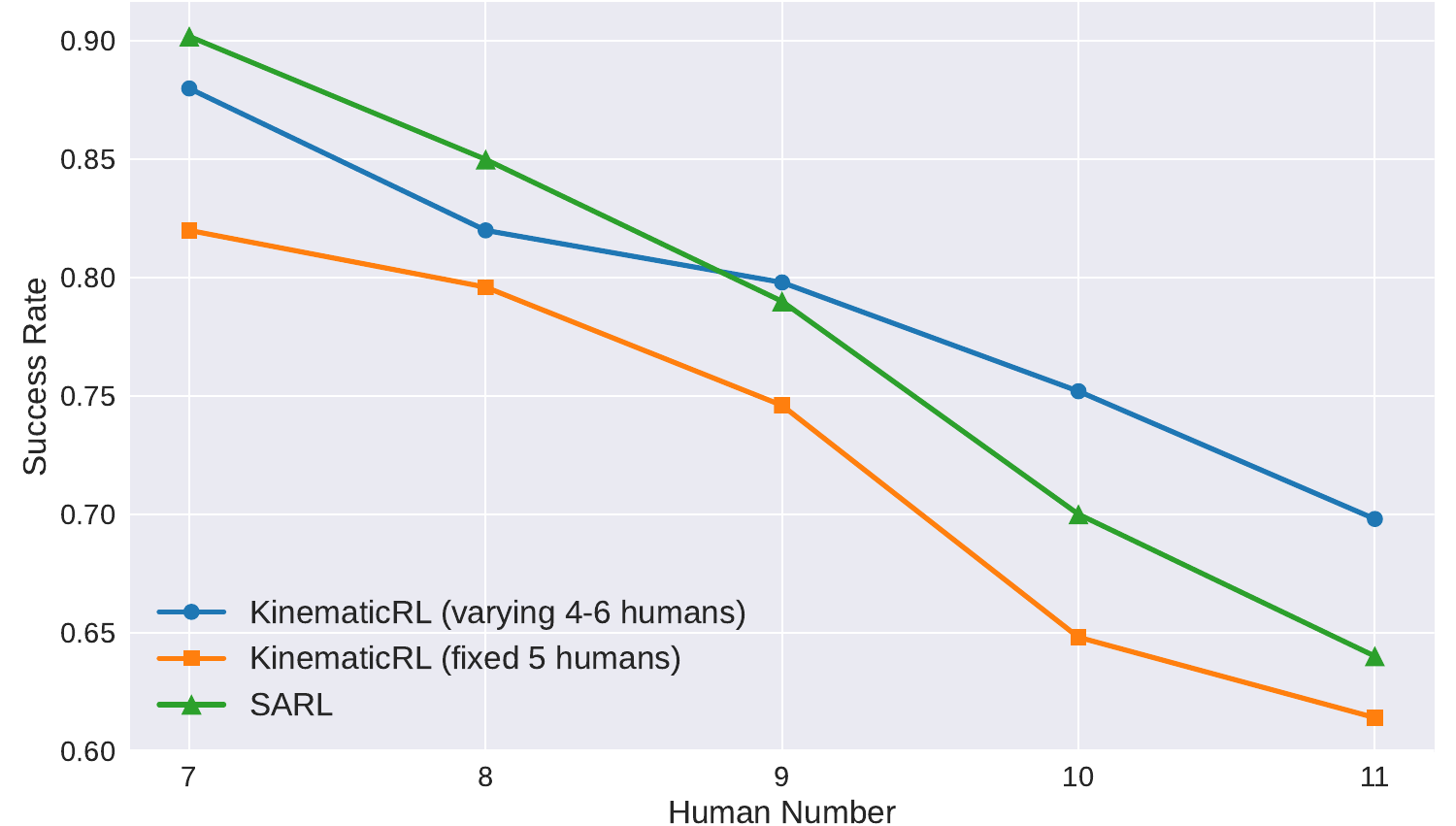}
    \caption{Performance of gated transformer across varying crowd densities. KinematicRL trained with varying 4-6 humans outperforms SARL when crowd size exceeds 9, whereas KinematicRL trained under fixed 5-human setting underperforms both methods. These results show self-attention scales more effectively than global attention in dense crowds and varying-cardinality training is crucial for generalization.}
    \label{fig:generalization-variable-density}
\end{figure}
Fig.~\ref{fig:generalization-variable-density} shows the tendency of success rate under varying crowd densities. Notably, KinematicRL trained with varying 4-6 humans outperforms SARL when human number exceeds 9, and decrease slower compared to SARL, indicating self-attention scales more effectively than global attention as crowd size increases. However, KinematicRL trained with a fixed 5-human setting underperforms both methods. As analyzed in Section~\ref{subsec:network-arch}, training with fixed number of humans creates a significant distribution shift that manifests as significant performance degradation when the model is queried on crowd densities outside the training distribution. Comparatively, training with varying number of humans enables the network to handle inputs of different lengths during training, thereby improving generalization.
\subsubsection{Robustness to Sensing Inaccuracies}
\begin{figure}[!ht]
    \centering
    \includegraphics[width=1.0\linewidth]{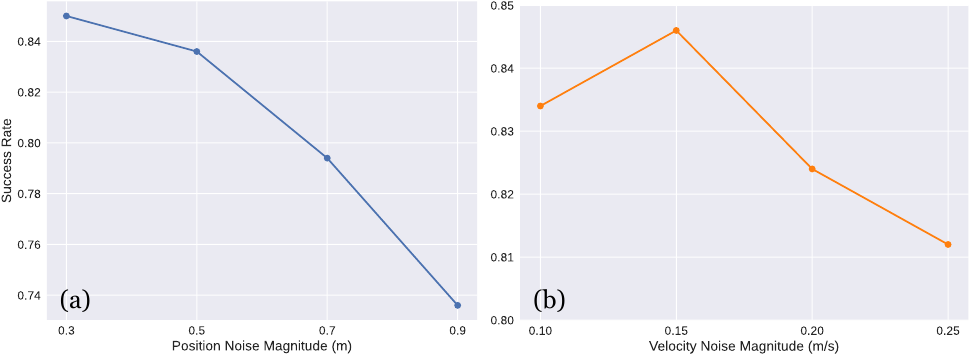}
    \caption{Robustness of gated transformer to position and velocity noise. The network remains robust for position noise up to 0.5m and across the full range of velocity noise tested. (a) Success rate under varying position noise magnitudes. (b) Success rate under varying velocity noise magnitudes.}
    \label{fig:net-noise-robustness}
\end{figure}
Within simulation, we assume access to perfect human positions and velocities, whereas in real world these measurements are subject to errors. To simulate this error, we train the policy under perfect sensing, and inject different noise magnitudes during test phase. Fig.~\ref{fig:net-noise-robustness} evaluates the robustness of our network under such setting. The network maintains high success rates for position errors up to 0.5m, with a significant drop at 0.7m, which is well beyond errors of current detectors such as Li2Former~\cite{li2former}. Performance remains robust across the full range of velocity noise tested. This is consistent with the position noise results: with a simulation timestep of 0.25s, a 0.3m position error corresponds to an unrealistic 1.2m/s velocity error for a tracking algorithm. Within the tested velocity noise range, the maximum induced position error is only 0.06m, which has negligible impact on network performance. In conclusion, the proposed architecture demonstrates robust performance across realistic ranges of position and velocity noise.

\subsubsection{Ablation Study}

We aim to answer how much variation is necessary for training gated transformer, along with the effectiveness of proposed network modules.
\paragraph{Scenario-Based Domain Randomization}
\begin{figure}[!ht]
    \centering
    \includegraphics[width=1.0\linewidth]{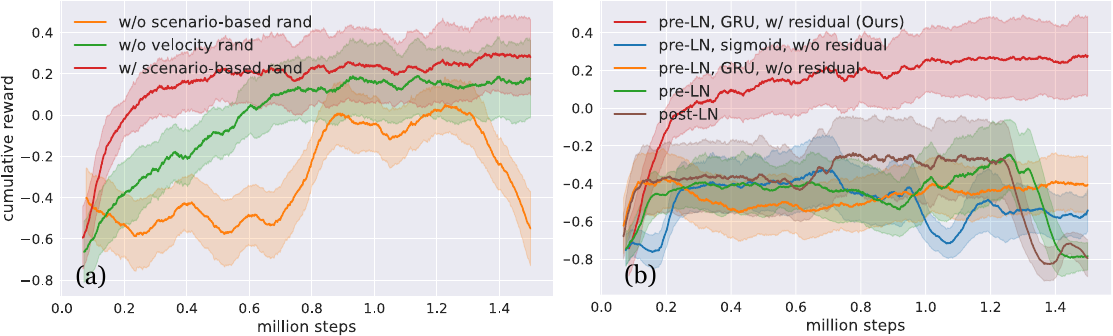}
    \caption{Ablation study. (a) Effectiveness of scenario-based domain randomization. The network fails to learn without sufficient randomization, and performance degrades when human velocities are not randomized. (b) Ablation of network design choices. Removing any of the three key design components---forward layer normalization, GRU-based gating, or modified residual connection---significantly hinders the learning process.}
    \label{fig:ablation-study}
\end{figure}
Earlier results showed that varying the number of humans is essential for generalization, and here we keep human number fixed to further ablate the effect of scenario-based domain randomization, which randomizes human positions, velocities and behavior patterns to increase state-space entropy. Ablation results (Fig.~\ref{fig:ablation-study}(a)) reveal that this increased stochasticity is indispensable for training transformers, as the limited variability in human motion provides insufficient cues for self-attention to differentiate among agents. The network learns fine without velocity randomization, but we find adding it leads to better final performance, likely because it produces more distinguishable motion patterns for self-attention to exploit.
\paragraph{Network Modules}
The key design components of our network include forward layer normalization (pre-LN), GRU-based gating, and modified residual connection (w/ residual). We ablate each of these components by comparing reward curves in Fig.~\ref{fig:ablation-study}(b). To offer a direct comparison with GRU, we design a simplified gating with sigmoid modulation on the output stream, which replaces Eq.~\eqref{eq:gru-forward-equation} with
\begin{equation}
h_{l+1}=x_{l}+ \sigma (W_{g}x_{l}+b_{g}) \odot h_{l}
\label{eq:simplified-output-modulation}
\end{equation}
where $\sigma$ is the sigmoid function. We refer to this baseline as ``pre-LN, sigmoid, w/o residual'', denoting a pre-LN Transformer with sigmoid gating and without modified residual connections. Bias term $b_{g}$ is initialized to $-5.0$ to enforce Markovian initialization. Since only our full model successfully learns to solve the task, we conclude that removing any of these three components hinders the learning process.
\begin{table}[!htb]
    \centering
    \caption{Ablation study on temporal encoder} \label{tab:net-ablation-temporal-encoder}
    \begin{tabular}{cccccc}
    \hline
    Structure & $SR$ & $CR$ & $\mu T$ & $DF$ \\ \hline
    Temporal First & 0.81 & 0.19  &  12.01$\scriptstyle \pm 4.10$ & \textbf{0.02}$\scriptstyle \pm 0.05$ \\
    Spatial Only & 0.78 & 0.22 & 9.90$\scriptstyle \pm 3.48$ & 0.05$\scriptstyle \pm 0.08$ \\ 
    \cellcolor[HTML]{EFEFEF}Spatial + Temporal (Ours)
    & \cellcolor[HTML]{EFEFEF}\textbf{0.84} 
    & \cellcolor[HTML]{EFEFEF}\textbf{0.15}
    & \cellcolor[HTML]{EFEFEF}\textbf{9.76}$\scriptstyle \pm 2.96$ 
    & \cellcolor[HTML]{EFEFEF}0.03$\scriptstyle \pm 0.06$  \\\hline
    \end{tabular}
\end{table}
\par We also report ablation results for the temporal encoder in Table~\ref{tab:net-ablation-temporal-encoder} (metrics are averaged over 500 test cases). The studied models include: (1) Temporal encoder before spatial encoder (``Temporal First''); (2) Replacing temporal encoder with another spatial encoder (``Spatial Only''). Performance improves when the spatial encoder precedes the temporal encoder, suggesting that stable spatial features provide a stronger foundation for extracting temporal relations. Using spatial encoders exclusively does achieve competitive performance, potentially because GRU suffices as a short-horizon memory mechanism. However, temporal encoder provides more scalable and explicit modeling of temporal information, which is reflected in higher success rate achieved by the full model.
\subsection{Tracking Pipeline Evaluation}
\begin{figure}[!htb]
    \centering
    \includegraphics[width=1.0\linewidth]{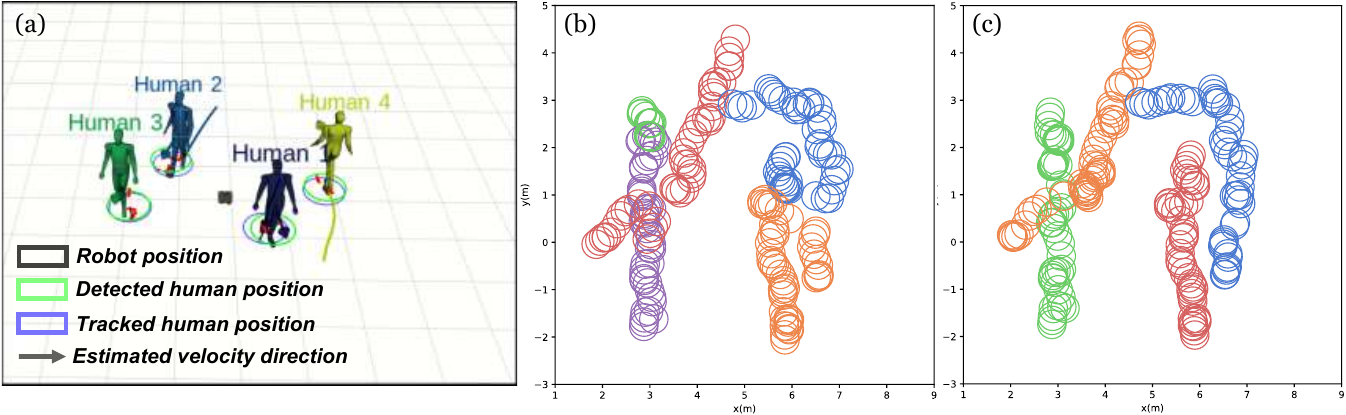}
    \caption{Tracking experiment setup and matching result under different velocity-similarity weights. Circles represent denoised positions of detected humans. Without the velocity term ($\alpha=0.0$), the tracker frequently swaps identities (e.g., blue and orange trajectories in (b)), whereas incorporating velocity similarity ($\alpha=0.4$) maintains consistent identities. (a) Scene with spatially proximate humans. (b) $\alpha=0.0$. (c) $\alpha=0.4$.}
    \label{fig:tracking-setup-result}
\end{figure}
Effectiveness of tracking pipeline is tested in simulation with 4 humans, with test scene 
containing occlusions and close spatial distance (example shown in Fig.~\ref{fig:tracking-setup-result}(a)). Random noise is added to human point cloud, and Li2Former~\cite{li2former} is utilized for human detection. Due to the presence of noise, each method 
are evaluated for 3 times. 
\begin{table}[!htb]
    \centering
    \caption{Mean squared error for velocity}\label{tab:mse-velocity}
    \begin{tabular}{c|c|c|c|c}
    \hline
    Configurations & $K(\cdot)$ & $Avg(K(\cdot))$ & $SG(K(\cdot))$ & $B(K(\cdot))$ \\ \hline
    $\alpha=0$ & 0.26$\pm$0.13 & \cellcolor[HTML]{EFEFEF} 0.18$\pm$0.07 & 2.37$\pm$1.47 & 0.99$\pm$0.42 \\ \hline
    $\alpha=0.4$ & 0.17$\pm$0.07 & \cellcolor[HTML]{EFEFEF} \textbf{0.11}$\pm$0.06 & 0.65$\pm$0.75 & 0.23$\pm$0.19 \\ \hline
    \end{tabular}
\end{table}
\subsubsection{Matching Robustness}
For matching, effect of considering temporal features is mainly investigated, with weight $\alpha=0$ representing matching purely with spatial similarity. 
Matching results are shown in Fig.~\ref{fig:tracking-setup-result}(b)(c). When omitting temporal features, algorithm struggles to distinguish humans that are spatially close (e.g., human 1 and human 4 in Fig.~\ref{fig:tracking-setup-result}(a)), which is evidenced by wrong swapping of blue and orange circle, overlapping positions of red and purple circles, and the creation of multiple clusters when tracking the same human. In contrast, when incorporating weights on velocities, algorithm effectively distinguishes 4 humans.
\subsubsection{Velocity Estimation Accuracy}
\begin{figure*}[!ht]
    \centering
    \begin{tabular}{@{}m{0.02\textwidth}@{}ccc@{}}
        & \textbf{Past Trajectory} & \textbf{RViz Visualization} & \textbf{Real World} \\[5pt]
        \hline \\[-1.5ex]
        \adjustbox{valign=c}{\begin{sideways}\makebox[0pt][c]{(a)}\end{sideways}} &
        \includegraphics[height=0.15\textheight,valign=c]{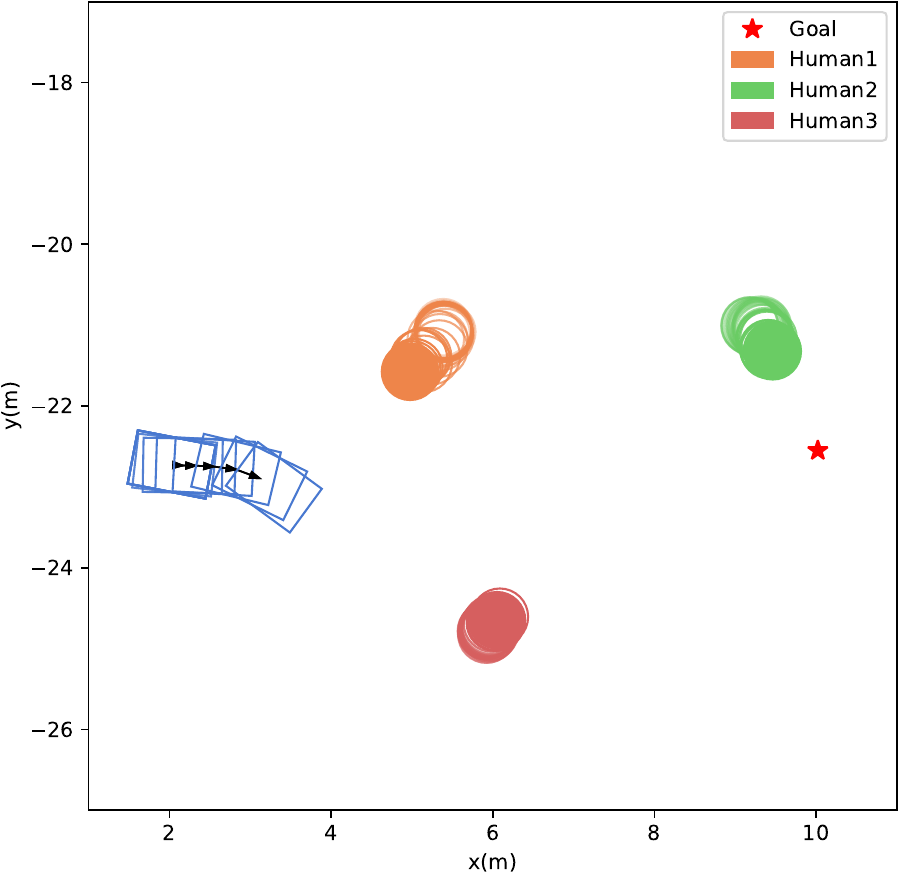} &
        \includegraphics[height=0.15\textheight,valign=c]{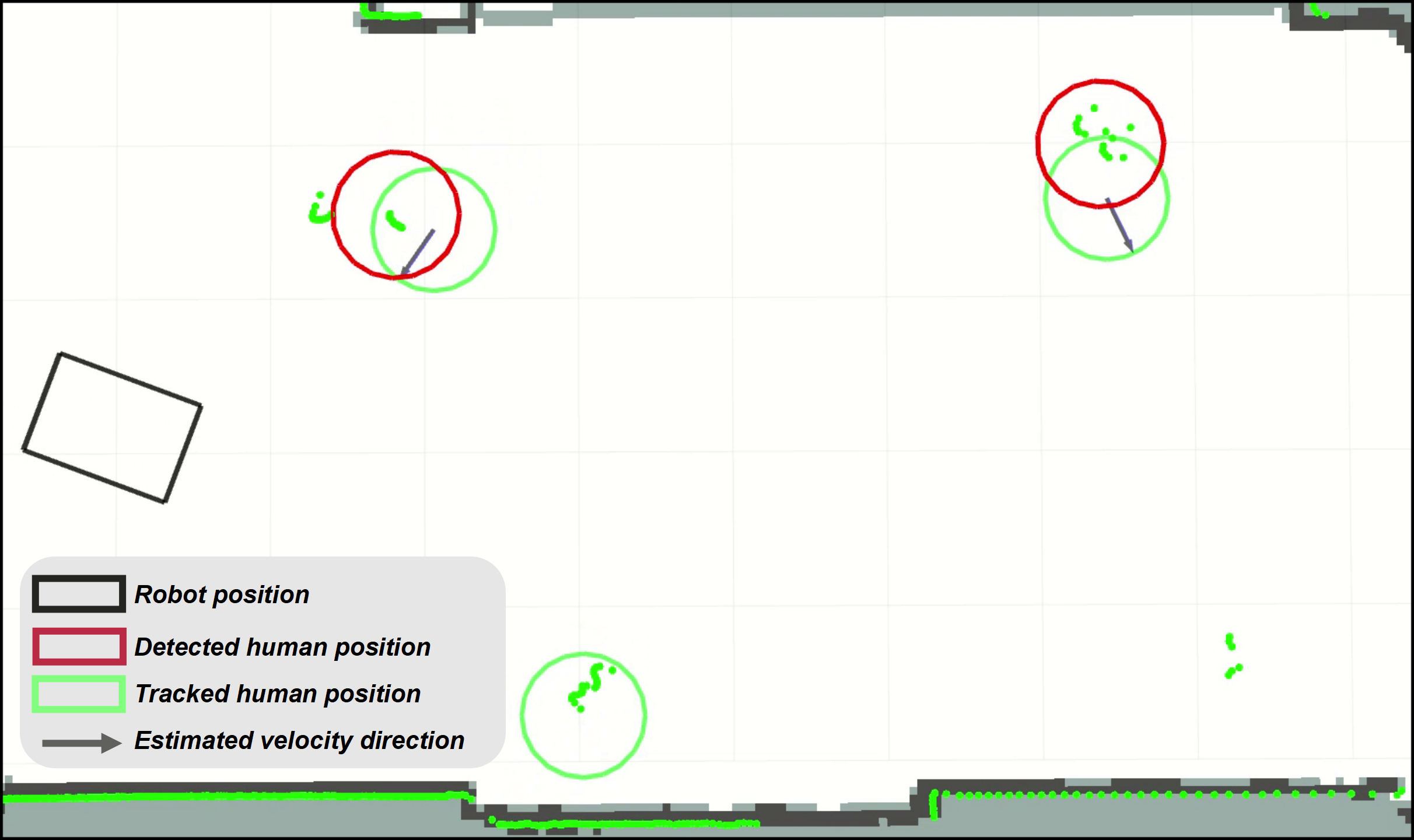} &
        \includegraphics[height=0.15\textheight,valign=c]{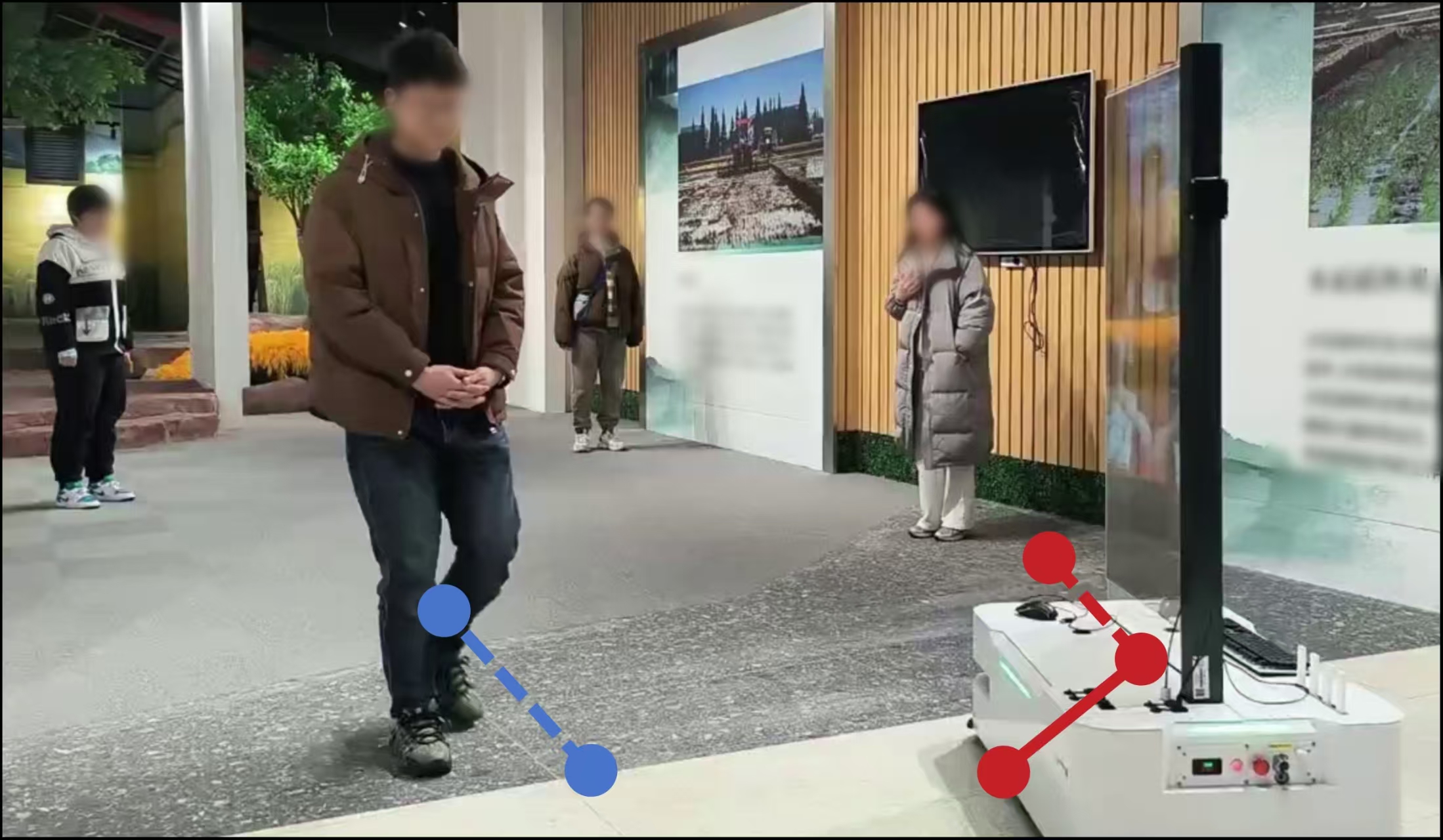} \\[20pt]
        \adjustbox{valign=c}{\begin{sideways}\makebox[0pt][c]{(b)}\end{sideways}} &
        \includegraphics[height=0.15\textheight,valign=c]{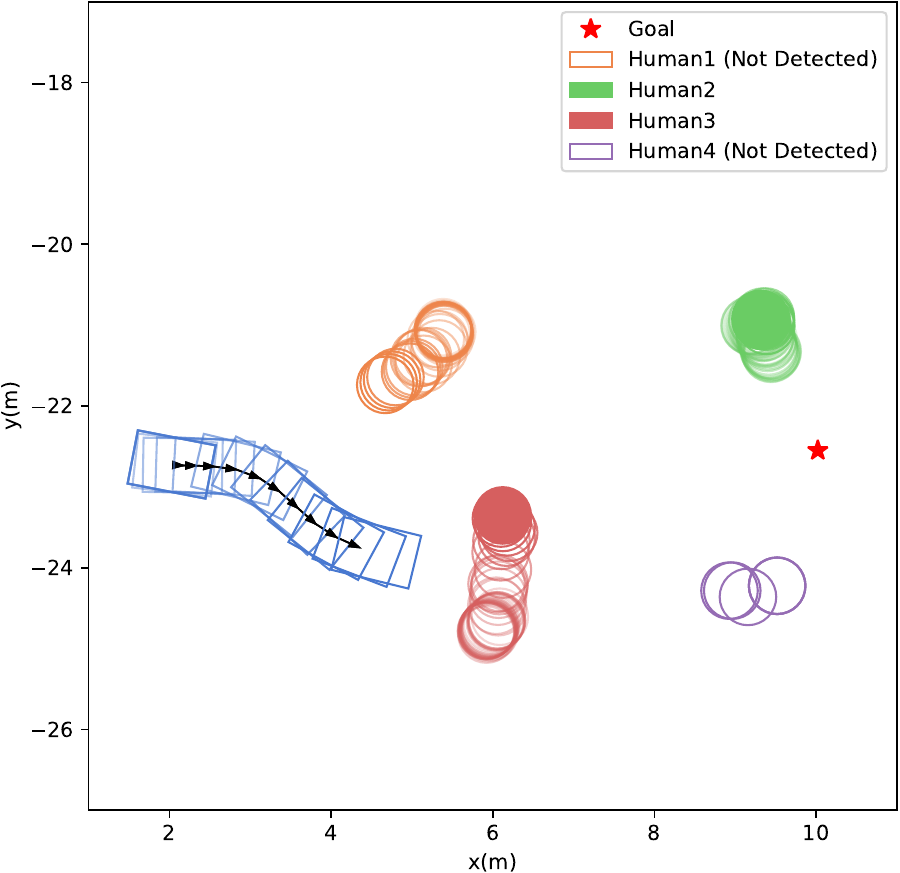} &
        \includegraphics[height=0.15\textheight,valign=c]{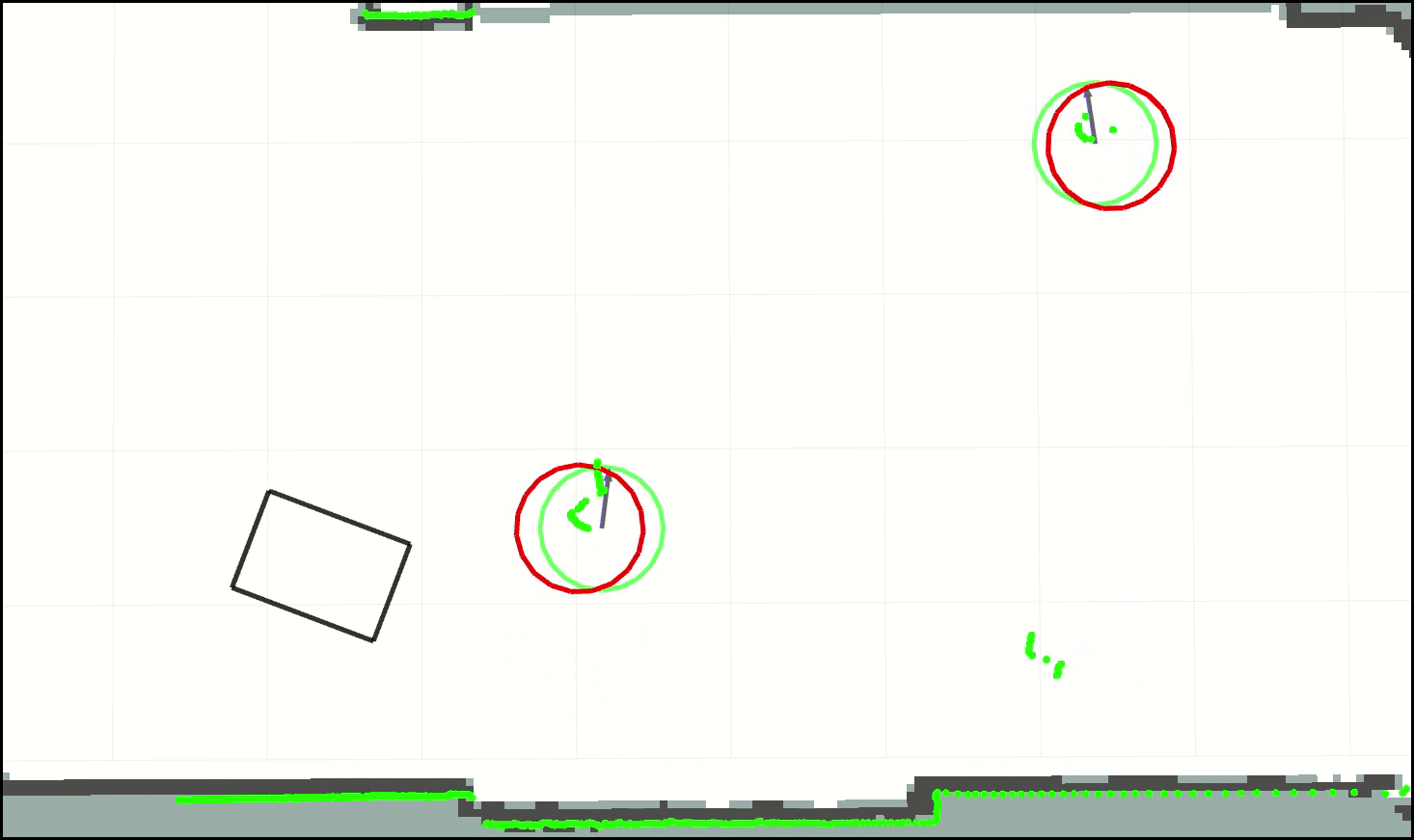} &
        \includegraphics[height=0.15\textheight,valign=c]{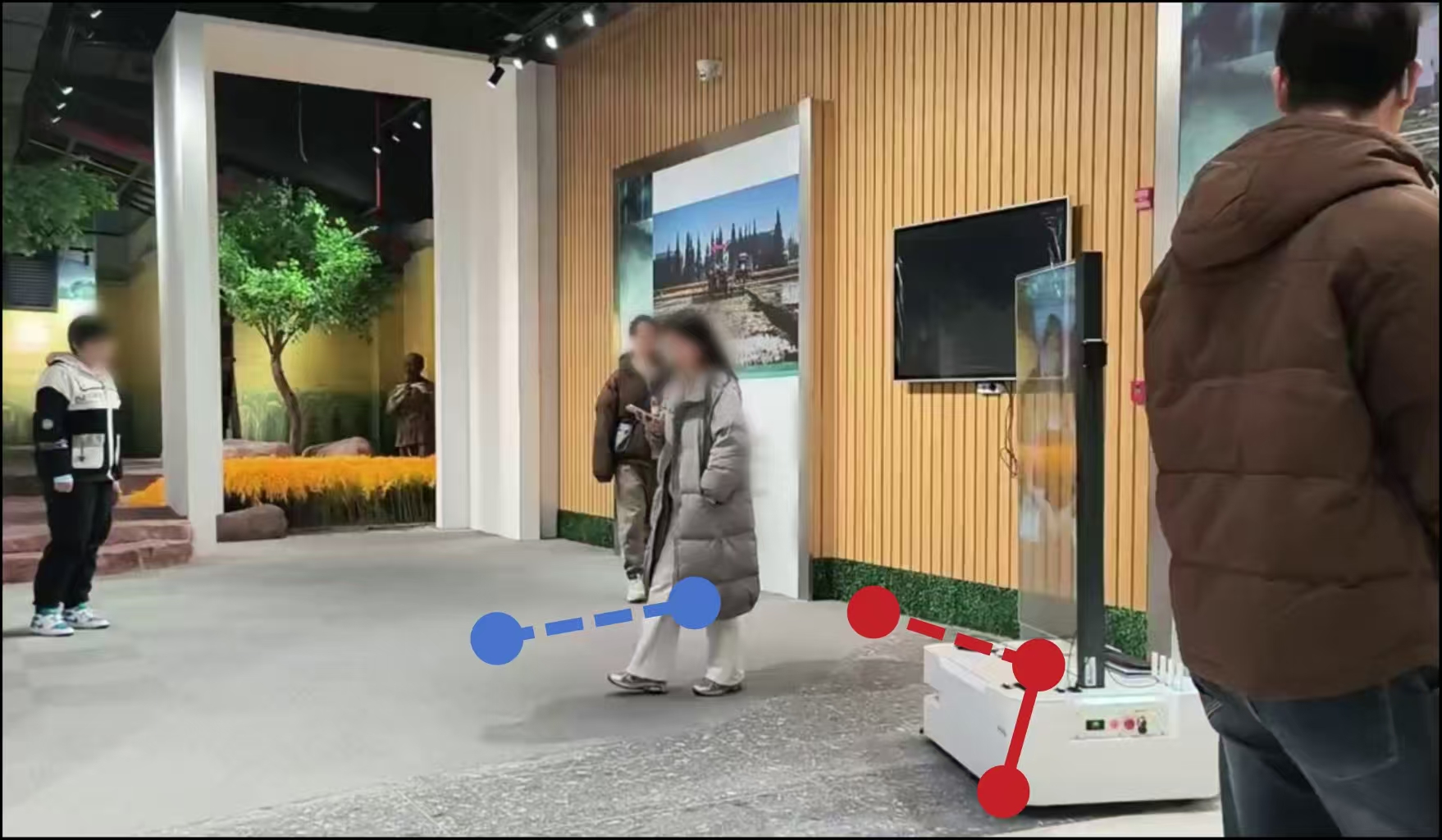} \\[20pt]
        \adjustbox{valign=c}{\begin{sideways}\makebox[0pt][c]{(c)}\end{sideways}} &
        \includegraphics[height=0.15\textheight,valign=c]{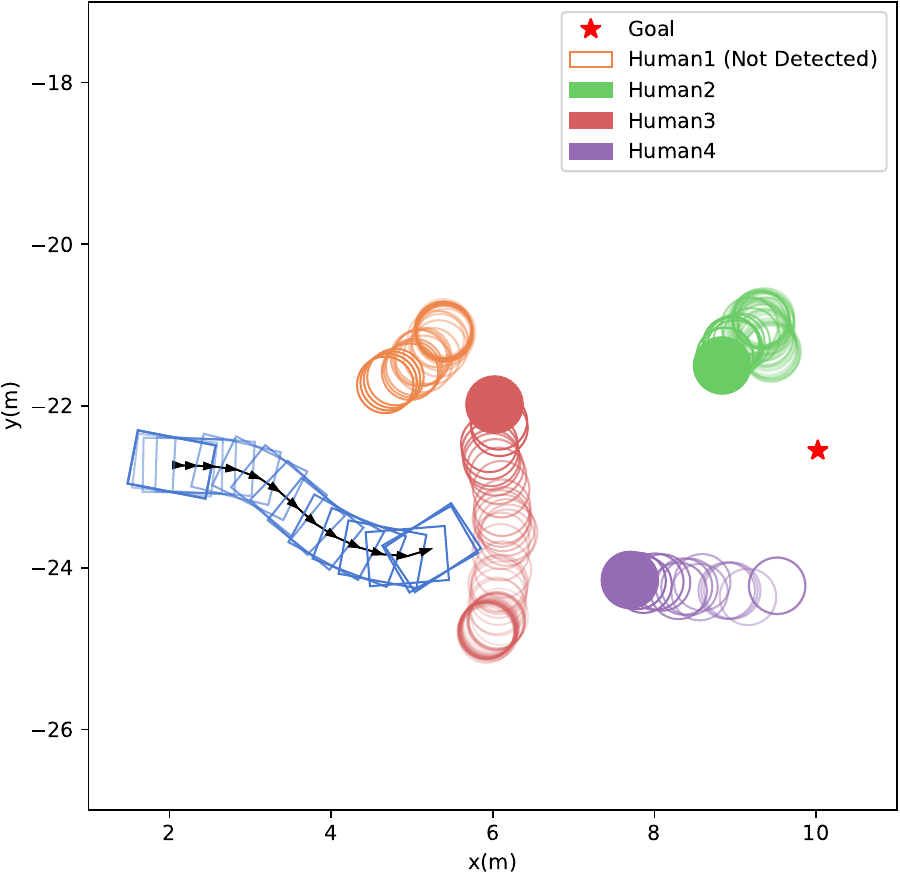} &
        \includegraphics[height=0.15\textheight,valign=c]{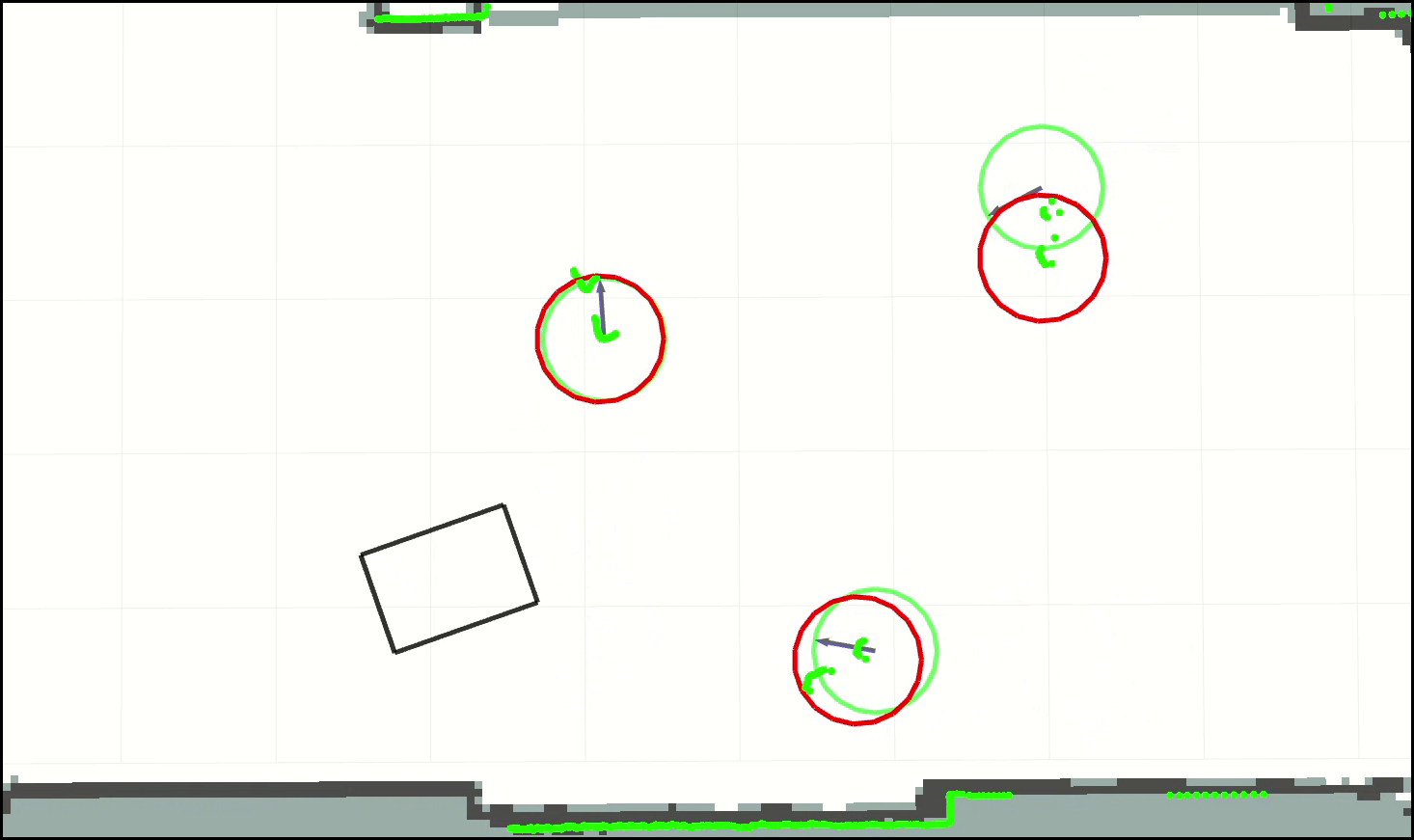} &
        \includegraphics[height=0.15\textheight,valign=c]{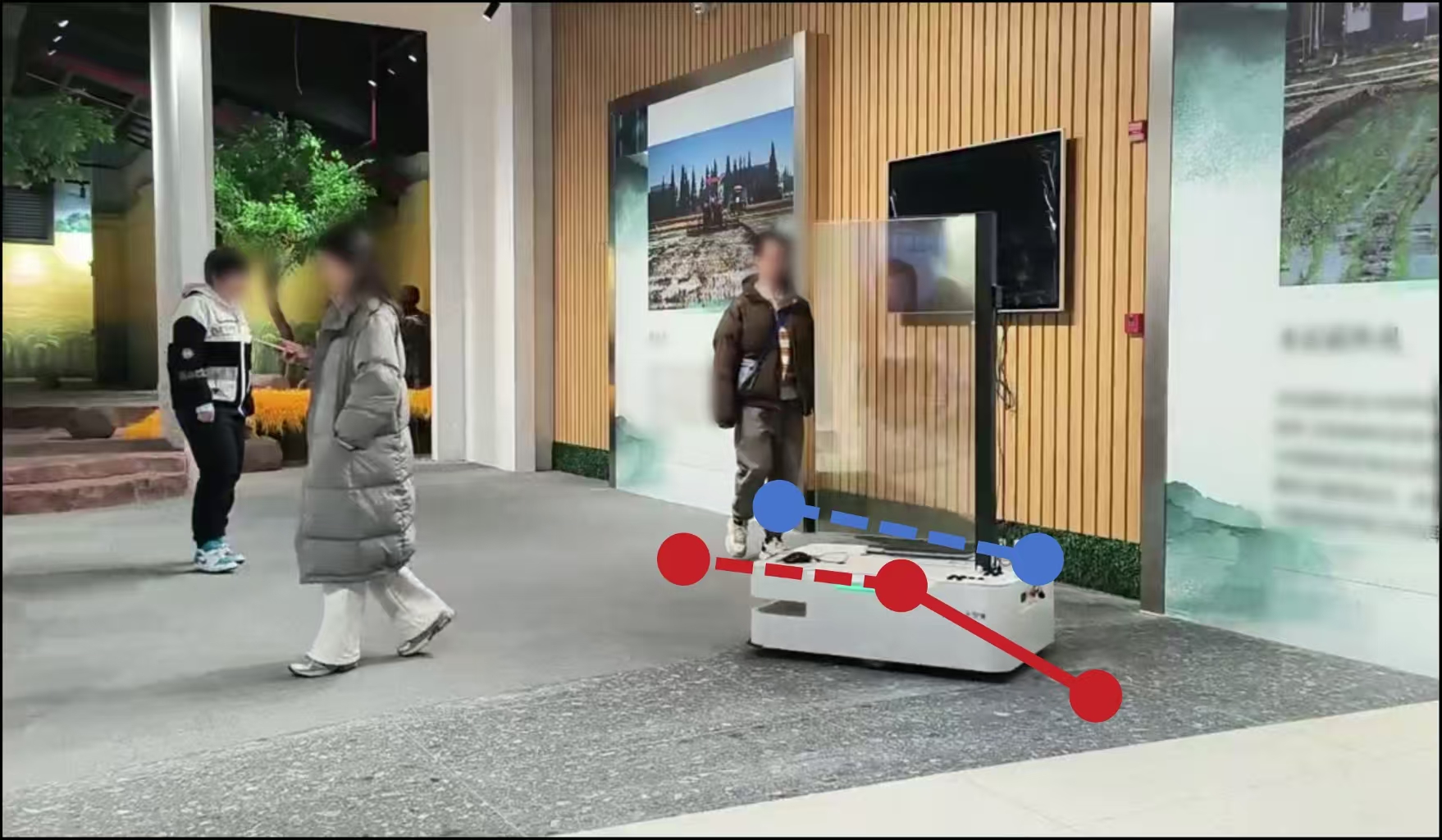} \\[20pt]
        \adjustbox{valign=c}{\begin{sideways}\makebox[0pt][c]{(d)}\end{sideways}} &
        \includegraphics[height=0.15\textheight,valign=c]{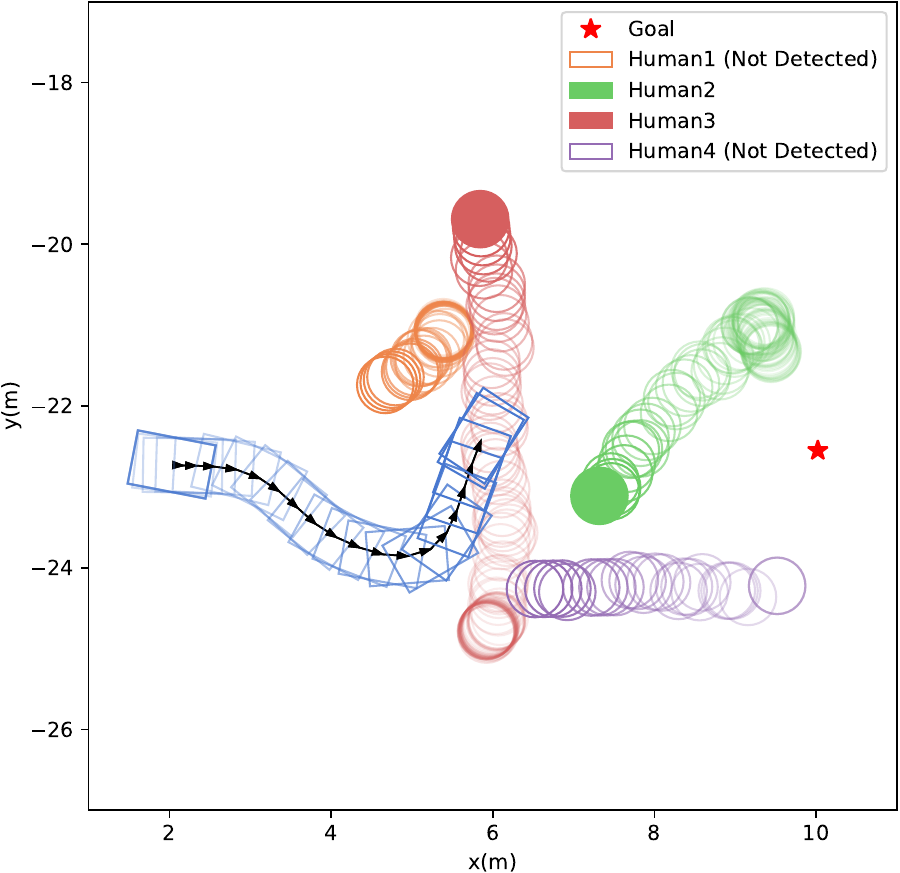} &
        \includegraphics[height=0.15\textheight,valign=c]{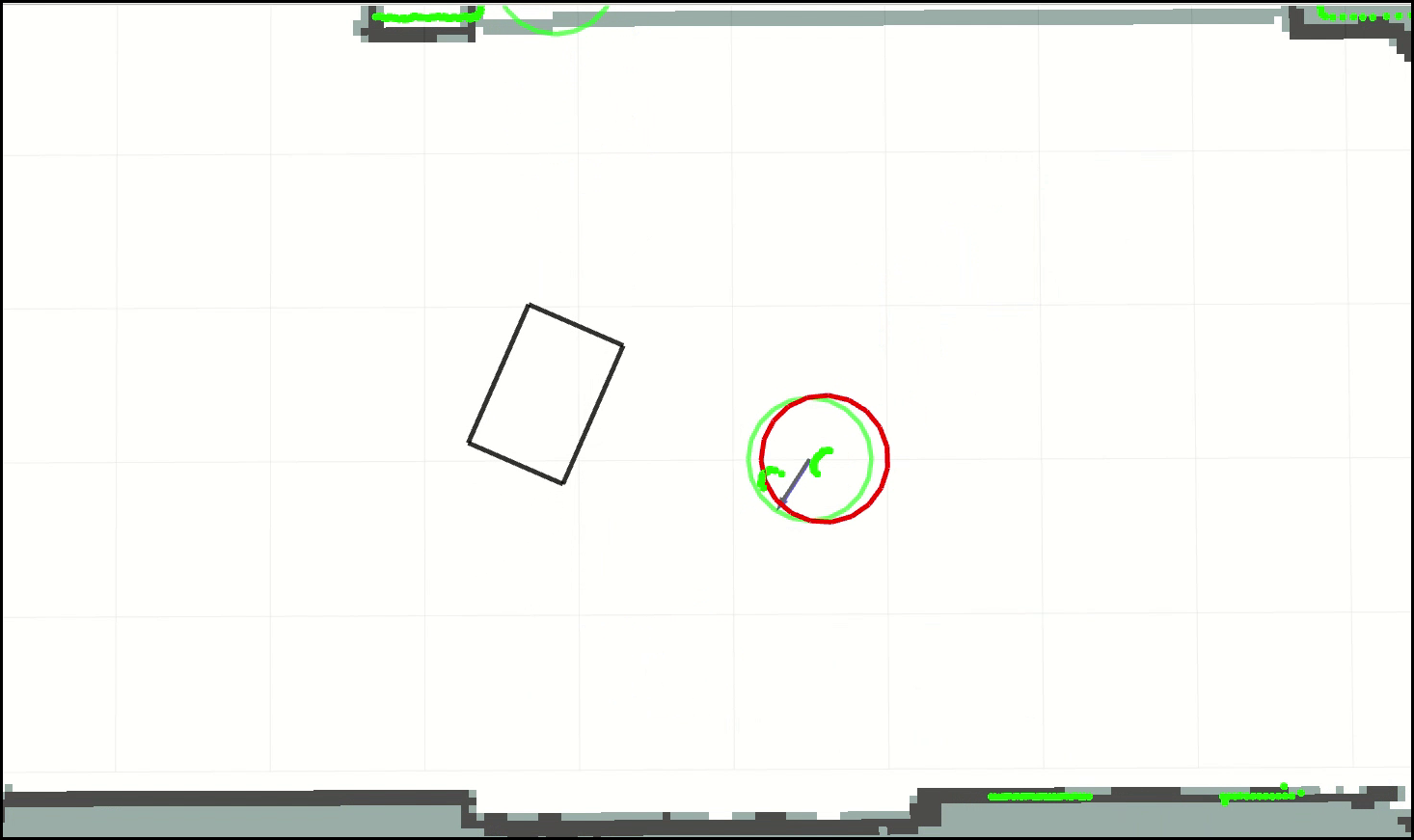} &
        \includegraphics[height=0.15\textheight,valign=c]{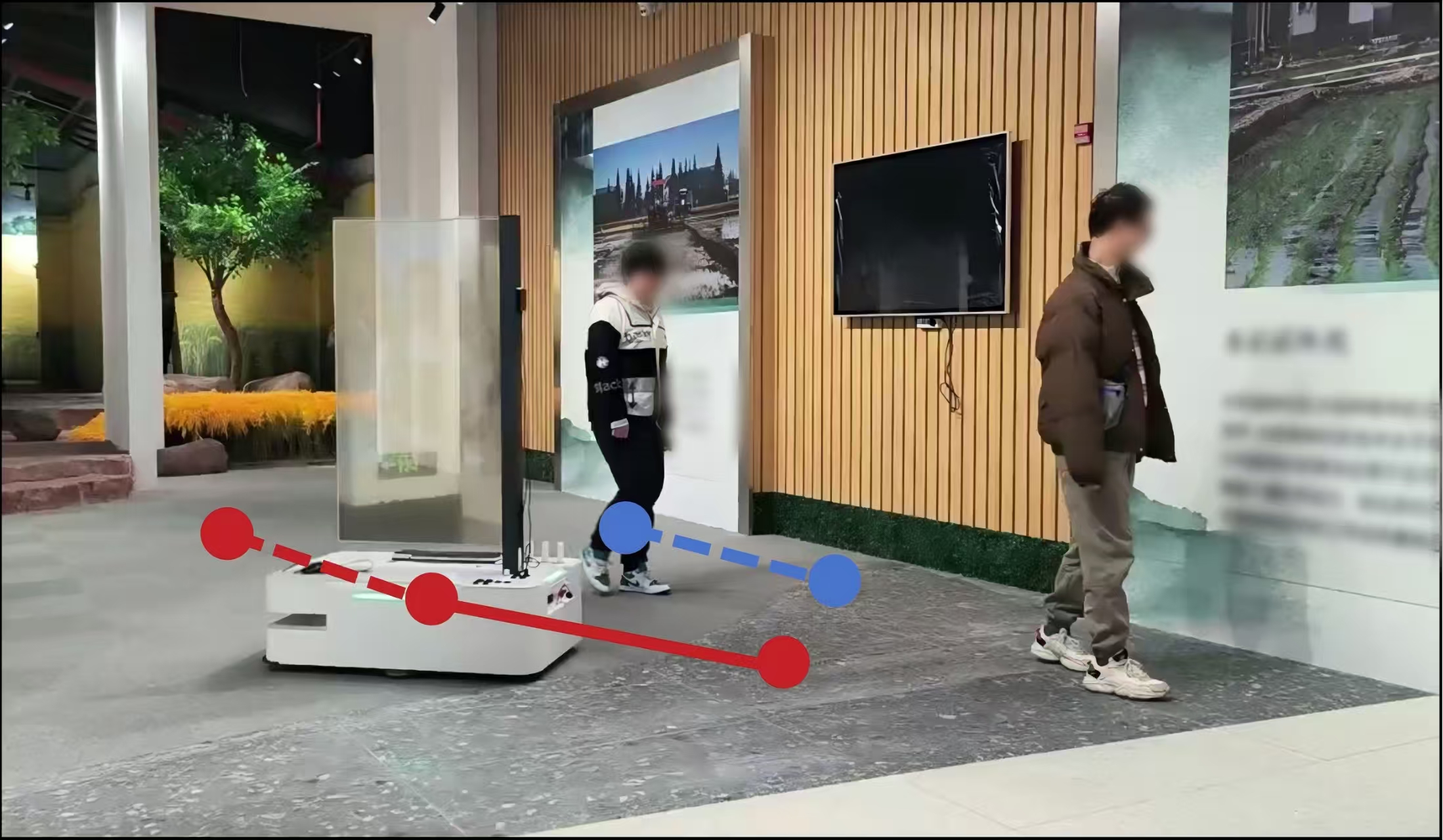} \\
    \end{tabular}
    \caption{Real world validation of the proposed model. Three columns represent past trajectory, RViz visualization and real world scenes, respectively. Each row illustrates interactions regarding each human. Real circles are Li2Former detections, and green circles are filtered results from tracking algorithm. Arrows indicate estimated orientations, and circles mark human and robot motion, with solid lines and dashed lines representing past and future motions, respectively. The robot demonstrates proactive and socially compliant collision-avoidance behaviors, successfully navigating past all humans.}
    \label{fig:real-world-experiments-sim-to-real}
\end{figure*}
Velocity estimation can be seen as a function mapping from observation to state, and function chosen are Kalman Filter (denote as $K(\cdot)$), Savitzky-Golay Filter (denote as $SG(\cdot)$), B\'{e}zier Regression (proposed by~\cite{bezier-predictor}, denote as $B(\cdot)$) and Historical Average Filter (our proposal in Eq.~\eqref{eq:weighted-historical-average}, 
denote as $Avg(\cdot)$). Mean squared error for velocity under different methods are shown in Table~\ref{tab:mse-velocity}. Our method achieves the lowest error of 0.11m/s, well within the acceptable range for network architecture (Fig.~\ref{fig:net-noise-robustness}). Further analysis of the results is provided below.
\paragraph{Polynomial Versus Historical Average}
Historical averaging produces significantly more accurate estimate compared with polynomial methods ($SG(\cdot)$ and $B(\cdot)$). The presence of noise in detected human positions often causes overlapping measurements, posing challenges for curve-fitting methods. Moreover, when curves are fitted to these noisy measurements, the velocity estimates derived from first-order derivatives become unreliable. For constrained optimization method like B\'{e}zier Regression, 
results tend to be better, but still underperforms the proposed method.
\paragraph{Improved Robustness with Velocity-Based Clustering}
For all estimation methods, using velocity-based matching reduces errors by improving data association. This highlights how maintaining accurate human position estimates is essential for reliable velocity estimation. Our proposal, which combines velocity-based matching with historical averaging, produces the best result, since it maintains historical human information and further reduces noise in velocity through averaging.

\subsection{Real-World Experiments}
\par Apart from simulation, we transfer the learned policy onto a real robot with differential constraints. The mobile robot is equipped with a WLR-716-Mini 2D LiDAR ($270^{\circ}$ scanning range), and an ARM Cortex-A57 (64-bit) processor operating at 1.43 GHz, which is responsible for executing the model as well as low-level action commands. Since this embedded platform does not include a GPU, inference of Li2Former is performed on a laptop with a GeForce RTX 2060 GPU at 17Hz (scan messages are published at 20Hz). We additionally deployed the entire system on a robot equipped with a Jetson Orin Nano 8GB and observed that, under a 30Hz scan rate, detection frequency can reach 26Hz. The tracking and planning modules introduce negligible additional latency, and operate at 10Hz on the robot. Empirically, this control frequency is adequate for mitigating jerk-related approximation errors, though our system allows higher control frequency. When no human is detected, robot navigates towards the goal with Proportional-Derivative (PD) control.
\par When training policy for real-world deployment, human numbers are randomized between 1 and 6, with preferred velocities uniformly sampled from $[0.3, 0.8]$ m/s to better approximate real-world scenario. Maximum linear and angular velocities, maximum linear and angular accelerations within simulation are configured to match the robot's hardware specifications. We choose conservative velocity and acceleration limits to address inertia-related challenges.
\par The test scene consists of four humans, with first and third human being cooperative, while others are non-cooperative. The robot exhibits collision avoidance maneuvers proactively, resulting in smooth trajectories. Due to limited sensing range and occlusions, number of detected humans varies across time frames, and the overall pipeline demonstrates robustness in this dynamic environment. Interactions regarding each human are shown in Fig.~\ref{fig:real-world-experiments-sim-to-real}. More real-world validations are available in our video presentation at \href{https://youtu.be/Z4af55A_xic}{\textcolor{blue}{https://youtu.be/Z4af55A\_xic}}.

%% file: contents/conclusion.tex
\section{Conclusion}\label{sec:conclusion}
\par In this work, a unified framework is presented to bridge the sim-to-real gap in social navigation through three core contributions: a higher-order control formulation with practical implementation, a LiDAR-only pedestrian tracking pipeline, and a gated spatio-temporal transformer architecture capable of adapting to time-varying crowd densities. The necessity of enforcing dynamic feasibility during training is highlighted, and both theoretical analysis and a practical training framework are provided to justify and enable second-order control for differential drive robots. A cluster-based tracking pipeline relying solely on 2D LiDAR is developed, employing a spatial-velocity distance metric and temporal aggregation to achieve robust data association and smooth velocity estimation. Furthermore, an unbiased residual gating block is introduced to balance reaction- and memory-based behaviors in social navigation, improving stability under variable crowd sizes. Extensive evaluations in both simulation and real-world environments demonstrate that the proposed policy, KinematicRL, achieves superior kinematic performance and can be deployed on physical robots with minimal adaptation.
\par Several limitations remain and motivate future work. First, although the proposed framework improves kinematic performance, collision rates increase in highly constrained scenarios, suggesting the need to incorporate explicit safety guarantees~\cite{second-order-cvar},~\cite{agile-but-safe}. Such approaches may also help address interactions with static obstacles and provide an additional layer of protection in highly occluded environments where Li2Former may produce uncertain detections. Second, the perception module introduces non-negligible computational overhead during deployment. Future work will investigate model compression techniques for Li2Former, such as quantization~\cite{qdetr}, to improve inference efficiency and facilitate deployment on resource-constrained edge platforms.

%% file: contents/appendix/appendix_tracking_error.tex
\renewcommand{\theequation}{A.\arabic{equation}}
\renewcommand\arraystretch{1.0}
\setcounter{equation}{0}
\setcounter{theorem}{0}
\section{Tracking error for order-n control}\label{appendix:tracking-error}
\begin{theorem}
Let control input $\mathbf{u}(t)=x^{(n)}(t)$ be piecewise constant 
with bounded step changes $\left\lVert \Delta \mathbf{u}\right\rVert <\infty$ at discrete times. 
Then tracking error for 
$x^{(k)}(t)$ over time step $\Delta t$ is $O\left(\frac{\Delta t^{n-k}}{(n-k)!}\right)$.
\end{theorem}
\begin{proof}
Let $\mathbf{u}_d(t)$ and $\mathbf{u}_n(t)$ denote the desired control command and the 
actual control input applied to the system over control period $\Delta t$, respectively. Note desired command $\mathbf{u}_d(t)$ is piecewise constant, and over control period 
$[0,\Delta t]$ this constant will be denoted as $\mathbf{u}_d$. 
For ideal model $\mathcal{M}_{\mathcal{I}}$ within simulation, desired input is achieved instantly, so state transition is given by:
\begin{equation}
\mathcal{M}_{\mathcal{I}}: \ \dot{\mathbf{x}}_{\mathcal{I}}(t)=\mathbf{A}_n\mathbf{x}_d(t)
+\mathbf{B}_n\mathbf{u}_d
\end{equation}
with the state-space matrices defined as:
\begin{equation}
\mathbf{A}_n = \begin{bmatrix}
0 & 1 & 0 & \cdots & 0 \\
0 & 0 & 1 & \cdots & 0 \\
\vdots & \vdots & \vdots & \ddots & \vdots \\
0 & 0 & 0 & \cdots & 1 \\
0 & 0 & 0 & \cdots & 0
\end{bmatrix}, \quad 
\mathbf{B}_n = \begin{bmatrix} 0 \\ 0 \\ \vdots \\ 0 \\ 1 \end{bmatrix}
\end{equation}
For real-world model $\mathcal{M}_{\mathcal{R}}$, state transition is given by:
\begin{equation}
\mathcal{M}_{\mathcal{R}}: \ \dot{\mathbf{x}}_{\mathcal{R}}(t)=\mathbf{A}_n\mathbf{x}_n(t)
+\mathbf{B}_n(\mathbf{u}_n(t)+d(t))
\end{equation}
Here $d(t)$ accounts for external disturbances such as friction, and is 
bounded by $D_{max}$. Then error dynamics with disturbance can be modeled as a 
first-order differential equation:
\begin{equation}
    \begin{aligned}
    \dot{\mathbf{e}}(t)&=\dot{\mathbf{x}}_{\mathcal{I}}(t)-\dot{\mathbf{x}}_{\mathcal{R}}(t)\\
    &= \mathbf{A}_n\mathbf{e}(t)+\mathbf{B}_n (\delta_u(t)+d(t))
    \end{aligned}
    \label{eq:error-ode}
\end{equation}
with $\mathbf{e}(t)$ being the state error and $\delta_u(t)$ being the control error. 
Without loss of generality, assume $\mathbf{e}(0)=\mathbf{0}$ and $\mathbf{u}_{n}(0)=\mathbf{0}$. Solution to ordinary differential Eq.~\eqref{eq:error-ode} at time $t\in [0,\Delta t]$ is:
\begin{equation}
\mathbf{e}(t) = \int_0^t e^{\mathbf{A}_n(t-\sigma)} 
\mathbf{B}_n (\delta_u(\sigma) + d(\sigma)) d\sigma
\end{equation}
Denote $t-\sigma$ as $\tau$, 
since $\mathbf{A}_n$ is nilpotent such that $\mathbf{A}_{n}^n=\mathbf{0}$, matrix exponential series $e^{\mathbf{A}_n\tau}\mathbf{B}_n$ evaluates to:
\begin{equation}
(\mathbf{I})\mathbf{B}_n + (\tau)\mathbf{A}_n \mathbf{B}_n+ \ldots+
\left(\frac{\tau^{n-1}}{(n-1)!}\right)\mathbf{A}_n^{n-1} \mathbf{B}_n
=\begin{bmatrix}
\frac{\tau^{n-1}}{(n-1)!} \\
\vdots \\
\tau \\
1
\end{bmatrix}
\end{equation}
Thus, error in the $k$-th derivative, $\mathbf{e}_k(t)$, can be calculated as:
\begin{equation}
\mathbf{e}_k(t) = \int_0^t \frac{(t-\sigma)^{n-k-1}}{(n-k-1)!} 
(\delta_u(\sigma) + d(\sigma)) d\sigma
\end{equation}
Evaluate the error at the end of the control period $t=\Delta t$, and applying 
triangle inequality, yields:
\begin{equation}
\left\lVert \mathbf{e}_k(t)\right\rVert \leq 
\int_0^{\Delta t} \frac{(\Delta t-\sigma)^{n-k-1}}{(n-k-1)!} 
\left(\left\lVert \delta_u(\sigma)\right\rVert  + 
\left\lVert d(\sigma)\right\rVert \right) d\sigma
\end{equation}
Solve the integral, then substitute the bounds for control input 
change and external disturbances, we obtain the final result:
\begin{equation}
\left\lVert \mathbf{e}_k(t)\right\rVert\leq 
\left(\sup \left\lVert \Delta \mathbf{u}\right\rVert
+D_{max}\right)\frac{(\Delta t)^{n-k}}{(n-k)!}=
O\left(\frac{(\Delta t)^{n-k}}{(n-k)!}\right)
\end{equation}
\end{proof}

\begin{remark}
Although our proof does not rely on the precise evolution of control error $\delta_u(t)$ and just require it to be bounded, it can be instructive to consider a concrete example. Assume actuator lag can be represented as a first-order system (low-pass filter):
\begin{equation}
\dot{\mathbf{u}}_n(t)=\frac{1}{\tau_a}\left(\mathbf{u}_d-\mathbf{u}_n(t)\right)
\label{eq:model-actuator-first-order-system}
\end{equation}
where $\tau_a$ is the actuator time constant, with a smaller $\tau_a$ representing 
a faster actuator. Under the initial conditions, solution to this ordinary differential equation is:
\begin{equation}
\delta_u(t)=\mathbf{u}_de^{-\frac{t}{\tau_a}},\quad t\in [0,\Delta t]
\end{equation}
Intuitively, this shows that control error decays exponentially over the control interval, but we choose to neglect it for the ideal model in simulation, causing the sim-to-real gap. This assumption becomes less problematic as control order $n$ increases, based on the above theorem.
\end{remark}

%% file: contents/appendix/appendix_maximum_entropy.tex
\section{Stochastic iLQR as Maximum Entropy Controller}\label{appendix:maximum-entropy-control}
\begin{theorem}
Under linear dynamics and quadratic cost, solution to problem~\eqref{eq:maximum-entropy-control} is a stochastic Gaussian policy:
\begin{equation}
    \mathbf{u}_t\sim\mathcal{N}(\hat{\mathbf{u}}_t
    +\mathbf{K}_{t}(\mathbf{x}_t-\hat{\mathbf{x}}_t)+\mathbf{k}_t,
    Q_{\mathbf{u},\mathbf{u}t}^{-1})
\end{equation}
where $\mathbf{K}_{t}=-Q_{\mathbf{u,u}t}^{-1}Q_{\mathbf{u,x}t},\mathbf{k}_{t}=-Q_{\mathbf{u,u}t}^{-1}Q_{\mathbf{u}t}$, with Q and value functions satisfying the following recurrence relation:
\begin{equation}
    \begin{aligned}
    Q_{\mathbf{xu}t}&=\ell_{\mathbf{xu}t}+
    f_{\mathbf{xu}t}^{\top}V_{\mathbf{x}t+1}+ 
    f_{\mathbf{xu}t}^{\top}V_{\mathbf{x},\mathbf{x}t+1}f_{ct} \\
    Q_{\mathbf{xu},\mathbf{xu}t}&=\ell_{\mathbf{xu},\mathbf{xu}t}+f_{\mathbf{xu}t}
    ^{\top}V_{\mathbf{x,x}t+1}f_{\mathbf{xu}t} \\
    V_{\mathbf{x}t}&=Q_{\mathbf{x}t}-Q_{\mathbf{u,x}t}^{\top}Q_{\mathbf{u,u}t}^{-1}
    Q_{\mathbf{u}t} \\
    V_{\mathbf{x,x}t}&=Q_{\mathbf{x,x}t}-Q_{\mathbf{u,x}t}^{\top}Q_{\mathbf{u,u}t}^{-1}
    Q_{\mathbf{u,x}t}
    \end{aligned}
\end{equation}
\end{theorem}
In order to prove Theorem~\ref{theorem:maximum-entropy-controller}, we first prove the following lemma on the form of the final solution.
\begin{lemma}\label{lemma:gaussian-distribution}
Solution $\pi$ for maximum entropy objective
\begin{equation}
    \min_{\pi} \ \sum_{t = 1}^{T} \mathbb{E}_{\pi}\left[ \ell(\mathbf{x}_t, \mathbf{u}_t) \right]
    - \mathcal{H}(\pi(\mathbf{u}_t | \mathbf{x}_t))
    \label{eq:solving-for-distribution}
\end{equation}
under linear-quadratic assumptions, is a Gaussian distribution.
\end{lemma}\label{lem:gaussian-distribution}
\begin{proof}
Since $\pi$ is a distribution, it satisfies:
\begin{equation}
\int \pi(\mathbf{u}_t | \mathbf{x}_t) d\mathbf{u}_t = 1
\end{equation}
Thus we construct the Lagrangian:
\begin{equation}
    \begin{aligned}
    \mathcal{L}(\pi, \lambda)&= \int \pi(\mathbf{u}_t | \mathbf{x}_t) 
    \ell(\mathbf{x}_t, \mathbf{u}_t) d\mathbf{u}_t \\ &+ 
    \int \pi(\mathbf{u}_t | \mathbf{x}_t) \log \pi(\mathbf{u}_t | \mathbf{x}_t) d\mathbf{u}_t 
    \\ &+ \lambda \left( \int \pi(\mathbf{u}_t | \mathbf{x}_t) d\mathbf{u}_t - 1 \right)
    \end{aligned}
    \label{eq:functional-Lagrangian}
\end{equation}
Take the functional derivative w.r.t distribution $\pi$ and set to zero:
\begin{equation}
\frac{\delta \mathcal{L}}{\delta \pi} = \ell(\mathbf{x}_t, \mathbf{u}_t) + 
1 + \log \pi(\mathbf{u}_t | \mathbf{x}_t) + \lambda \coloneqq 0
\end{equation}
Thus $\pi(\mathbf{u}_t | \mathbf{x}_t) \propto \exp \left( - \ell(\mathbf{x}_t, \mathbf{u}_t) \right)$. 
Since cost function $\ell$ is quadratic, this indicates the optimal 
distribution is Gaussian. 
\end{proof}
Building upon Lemma~\ref{lemma:gaussian-distribution}, we now prove original Theorem~\ref{theorem:maximum-entropy-controller}.
\begin{proof}
As proven in Lemma~\ref{lem:gaussian-distribution}, solution to Eq.~\eqref{eq:maximum-entropy-control} is a Gaussian distribution. Assume
\begin{equation}
    \pi(\delta\mathbf{u}_t|\delta\mathbf{x}_t)=\mathcal{N}(\mu_t,\Sigma_t)
\end{equation}
Entropy of this distribution is given by: 
\begin{equation}
\mathbb{E}_\pi\left[-\log \pi\left(\delta\mathbf{u}_t |\delta\mathbf{x}_t\right)\right]=
\frac{1}{2}\log |\Sigma_t|+\text{const}
\end{equation}
With the additional entropy term, cost function is augmented as:
\begin{equation}
    \begin{aligned}
    \tilde{\ell}(\mathbf{x}_t,\mathbf{u}_t)&=\mathbb{E}_{\pi}\left[\ell \left(\mathbf{x}_t, \mathbf{u}_t\right)\right]
    -\mathcal{H}\left(\pi\left(\delta\mathbf{u}_t|\delta\mathbf{x}_t\right)\right) \\
    &=\ell (\hat{\mathbf{x}}_t,\hat{\mathbf{u}}_t) + 
    \frac{1}{2}\delta \mathbf{x}_t^{\top} \ell_{\mathbf{x},\mathbf{x}t}\delta \mathbf{x}_t +  \ell_{\mathbf{x}t}^{\top}\delta \mathbf{x}_t \\
    &+\frac{1}{2}\mathbb{E}_{\pi}\left[\delta \mathbf{u}_t^{\top} \ell_{\mathbf{u}, \mathbf{x}t}
    \delta \mathbf{x}_t+\delta \mathbf{x}_t^{\top} \ell_{\mathbf{x}, \mathbf{u}t} \delta \mathbf{u}_t\right]\\
    &+\mathbb{E}_{\pi}\left[\ell_{\mathbf{u}t}^{\top}\delta \mathbf{u}_t\right]+\frac{1}{2}\mathbb{E}_{\pi}\left[
    \delta \mathbf{u}_t^{\top} \ell_{\mathbf{u}, \mathbf{u}t} \delta \mathbf{u}_t \right]\\
    &-\frac{1}{2}\log |\Sigma_t| + \text{const}
    \end{aligned}
\end{equation}
Given $\delta\mathbf{u}_t\sim\mathcal{N}(\mu_t,\Sigma_t)$, expectations 
can be analytically computed as:
\begin{equation}
    \begin{aligned}
        \mathbb{E}_{\pi}&\left[\delta \mathbf{u}_t^{\top} \ell_{\mathbf{u}, \mathbf{x}t}
        \delta \mathbf{x}_t\right]
        =\mu_t^{\top}\ell_{\mathbf{u}, \mathbf{x}t}
        \delta \mathbf{x}_t \\
        \mathbb{E}_{\pi}&\left[\delta \mathbf{u}_t^{\top} \ell_{\mathbf{u}t}\right]=
        \mu_t^{\top}\ell_{\mathbf{u}t} \\
        \mathbb{E}_{\pi}&\left[\delta \mathbf{u}_t^{\top}
        \ell_{\mathbf{u}, \mathbf{u}t} \delta \mathbf{u}_t\right]=\operatorname{tr}
        \left(\Sigma_t\ell_{\mathbf{u},\mathbf{u}t}\right)+
        \mu_t^{\top}\ell_{\mathbf{u},\mathbf{u}t} \mu_t
    \end{aligned}
\end{equation}
For notational convenience, define
\begin{equation}
    \begin{aligned}
    \mathbf{C}_t&\coloneqq
    \begin{bmatrix}
    \ell_{\mathbf{x},\mathbf{x}t} & \ell_{\mathbf{x},\mathbf{u}t} \\
    \ell_{\mathbf{u},\mathbf{x}t} & \ell_{\mathbf{u},\mathbf{u}t}
    \end{bmatrix}&
    \mathbf{c}_t&\coloneqq
    \begin{bmatrix}
    \ell_{\mathbf{x}t} \\ \ell_{\mathbf{u}t}
    \end{bmatrix} \\
    \mathbf{Q}_t&\coloneqq
    \begin{bmatrix}
    Q_{\mathbf{x},\mathbf{x}t} & Q_{\mathbf{x},\mathbf{u}t} \\
    Q_{\mathbf{u},\mathbf{x}t} & Q_{\mathbf{u},\mathbf{u}t}
    \end{bmatrix}&
    \mathbf{q}_t&\coloneqq
    \begin{bmatrix}
    Q_{\mathbf{x}t} \\ Q_{\mathbf{u}t}
    \end{bmatrix} \\
    \mathbf{F}_t&\coloneqq 
    \begin{bmatrix}
    \mathbf{f}_{\mathbf{x}t} \\ \mathbf{f}_{\mathbf{u}t}
    \end{bmatrix} & 
    \mathbf{f}_t&\coloneqq f(\hat{\mathbf{x}}_t, \hat{\mathbf{u}}_t) - \hat{\mathbf{x}}_{t+1} \\
    \mathbf{V}_t&\coloneqq V_{\mathbf{x},\mathbf{x}t} & \mathbf{v}_t&\coloneqq V_{\mathbf{x}t} \\
    \end{aligned}
\end{equation}
And further denote 
$h(\Sigma_t)\coloneqq\frac{1}{2}\operatorname{tr}\left(\Sigma_t\ell_{\mathbf{u},\mathbf{u}t}\right)-\frac{1}{2}\log|\Sigma_t|$ the 
cost term associated with covariance $\Sigma_t$. Under linear dynamics and quadratic cost, it can be shown 
that V and Q function is also quadratic in $\delta \mathbf{u}_t$. Further incorporating the augmented cost and 
replace $\delta \mathbf{u}_t$ with mean $\mu_t$, we prove that Q and V take the form of:
\begin{equation}
    \begin{aligned}
    Q(\delta\mathbf{x}_t, \delta\mathbf{u}_t) &= \text{const}+
    \frac{1}{2}
    \begin{bmatrix}
    \delta \mathbf{x}_t \\
    \mu_t
    \end{bmatrix}^{\top}\mathbf{Q}_t
    \begin{bmatrix}
    \delta \mathbf{x}_t \\
    \mu_t
    \end{bmatrix} \\ &+
    \begin{bmatrix}
    \delta\mathbf{x}_t \\ \mu_t
    \end{bmatrix}^{\top}\mathbf{q}_t
    + h(\Sigma_t) \\
    V(\delta\mathbf{x}_t)&=\text{const}+
    \frac{1}{2}
    \delta \mathbf{x}_t^{\top}\mathbf{V}_t
    \delta \mathbf{x}_t \\ &+
    \delta\mathbf{x}_t^{\top}\mathbf{v}_t
    + h(\Sigma_t)
    \end{aligned}
    \label{eq:Q-V-functions}
\end{equation}
At discrete time step $t$, choosing distribution $\mathcal{N}(\mu_t,\Sigma_t)$ to minimize Q function 
in Eq.~\eqref{eq:Q-V-functions} requires taking derivative w.r.t $\mu_t$ and $\Sigma_t$:
\begin{equation}
    \begin{aligned}
    \nabla_{\mu}Q(\delta\mathbf{x}_t, \delta\mathbf{u}_t)&=Q_{\mathbf{u},\mathbf{x}t}\delta \mathbf{x}_t
    +Q_{\mathbf{u},\mathbf{u}t}\mu_t+Q_{\mathbf{u}t}\coloneqq 0 \\
    \nabla_{\Sigma}Q(\delta\mathbf{x}_t, \delta\mathbf{u}_t)&=\frac{1}{2}Q_{\mathbf{u},\mathbf{u}t}^{\top}
    -\frac{1}{2}\Sigma_t^{-1}\coloneqq 0
    \end{aligned}
\end{equation}
which shows 
\begin{equation}
\mu_t=\mathbf{K}_t\delta \mathbf{x}_t+\mathbf{k}_t,\Sigma_t=Q_{\mathbf{u,u}t}^{-1}
\label{eq:linear-control-law}
\end{equation}
further indicating that the time-varying linear Gaussian policy 
\begin{equation}
\mathbf{u}_t\sim\mathcal{N}(\hat{\mathbf{u}_{t}}
+\mathbf{K}_{t}(\mathbf{x}_t-\hat{\mathbf{x}}_t)+\mathbf{k}_t,
Q_{\mathbf{u},\mathbf{u}t}^{-1})
\end{equation}
is the solution to problem~\eqref{eq:maximum-entropy-control}. 
\par To establish the recurrence relation, at discrete time step $t$, rewrite Q function as:
\begin{equation}
    \begin{aligned}
        Q(\mathbf{\delta x}_t, \mathbf{\delta u}_t) &= \text{const}+
        \frac{1}{2}
            \begin{bmatrix}
            \delta \mathbf{x}_t \\
            \mu_t
            \end{bmatrix}^{\top}\mathbf{C}_t
            \begin{bmatrix}
            \delta \mathbf{x}_t \\
            \mu_t
            \end{bmatrix} \\ &+
            \begin{bmatrix}
            \delta\mathbf{x}_t \\ \mu_t
            \end{bmatrix}^{\top}\mathbf{c}_t
            +V(\delta \mathbf{x}_{t+1}) + h(\Sigma_t)
    \end{aligned}
    \label{eq:Q-fn-original-form}
\end{equation}
value function $V(\delta \mathbf{x}_{t+1})$ can be obtained by utilizing 
transition equation 
\begin{equation}
\delta\mathbf{x}_{t+1}=\mathbf{F}_t
\begin{bmatrix}
    \delta \mathbf{x}_t \\ \mu_t
\end{bmatrix} + \mathbf{f}_t
\end{equation}
which indicates that
\begin{equation}
    \begin{aligned}
    V(\delta \mathbf{x}_{t+1})&=\text{const}+
    \frac{1}{2}
    \begin{bmatrix}
    \delta \mathbf{x}_t \\
    \mu_t
    \end{bmatrix}^{\top}\mathbf{F}_t^{\top}\mathbf{V}_{t+1}\mathbf{F}_t
    \begin{bmatrix}
    \delta \mathbf{x}_t \\
    \mu_t
    \end{bmatrix} \\ &+
    \begin{bmatrix}
    \delta\mathbf{x}_t \\ \mu_t
    \end{bmatrix}^{\top}\mathbf{F}_t^{\top}\mathbf{V}_{t+1}\mathbf{f}_t
    +\begin{bmatrix}
    \delta\mathbf{x}_t \\ \mu_t
    \end{bmatrix}^{\top}\mathbf{F}_t^{\top}\mathbf{v}_{t+1}
    +h(\Sigma_t)
    \end{aligned}
    \label{eq:value-function}
\end{equation}
comparing Eq.~\eqref{eq:value-function} with expression in Eq.~\eqref{eq:Q-V-functions} shows that 
\begin{equation}
    \begin{aligned}
        \mathbf{Q}_t&=\mathbf{C}_t+\mathbf{F}_t^{\top}\mathbf{V}_{t+1}\mathbf{F}_t \\
        \mathbf{q}_t&=\mathbf{c}_t+\mathbf{F}_t^{\top}\mathbf{V}_{t+1}\mathbf{f}_t+\mathbf{F}_t^{\top}\mathbf{v}_{t+1}
    \end{aligned}
\end{equation}
which aligns with previous Eq.~\eqref{eq:standard-ilqr-backward-pass}. 
Furthermore, plugging in the linear control law~\eqref{eq:linear-control-law} into Q function 
and compare with quadratic form in~\eqref{eq:Q-V-functions}, we prove that
\begin{equation}
    \begin{aligned}
    \mathbf{V}_t&=Q_{\mathbf{x,x}t}-Q_{\mathbf{x,u}t}Q_{\mathbf{u,u}t}^{-1}Q_{\mathbf{u,x}t}\\
    \mathbf{v}_t&=Q_{\mathbf{x}t}-Q_{\mathbf{x,u}t}Q_{\mathbf{u,u}t}^{-1}Q_{\mathbf{u}t}
    \end{aligned}
\end{equation}
which is the same as previous recurrence relation~\eqref{eq:standard-ilqr-backward-pass}.
\end{proof}